\documentclass[a4paper,fleqn]{cas-dc}

\newcommand{\revOne}{black}

\usepackage{amsmath}
\usepackage{physics}
\usepackage{tikz}
\usepackage{color}
\usepackage{array}
\usepackage{multirow}
\usepackage{tabularx}
\usepackage{extarrows}
\usepackage{booktabs}
\usetikzlibrary{fadings}
\usetikzlibrary{patterns}
\usetikzlibrary{shadows.blur}
\usetikzlibrary{shapes}
\usetikzlibrary{external,calc,patterns,decorations.pathmorphing,decorations.markings,decorations.text,arrows,shapes,positioning, backgrounds}
\usepackage{soul}
\usepackage{caption}
\usepackage{subcaption}

\usepackage[authoryear]{natbib}

\usepackage{amsthm}
\newtheorem{rem}{Remark}
\newtheorem{thm}{Theorem}

\usepackage[switch]{lineno} 
\let\oldequation\equation
\let\oldendequation\endequation

\let\oldalign\align
\let\oldendalign\endalign

\renewenvironment{equation}
  {\linenomathNonumbers\oldequation}
  {\oldendequation\endlinenomath}
\renewenvironment{align}
  {\linenomathNonumbers\oldalign}
  {\oldendalign\endlinenomath}
  
\begin{document}
\let\WriteBookmarks\relax
\def\floatpagepagefraction{1}
\def\textpagefraction{.001}

\newcommand{\titlecmd}{Guaranteed Rejection-free Sampling Method Using Past Behaviours for Motion Planning of Autonomous Systems}

\shorttitle{\titlecmd}
\shortauthors{Enevoldsen and Galeazzi}
\title [mode = title]{\titlecmd}  

\author[1]{Thomas T. Enevoldsen}[
orcid=0000-0001-9855-7066
]
\cormark[1]
\ead{thomas.t.enevoldsen@dk.abb.com}

\affiliation[1]{organization={ABB A/S, Marine and Ports},%Department and Organization
            city={Ballerup},
            postcode={DK 2750}, 
            country={Denmark}}

\affiliation[2]{organization={Control, Robotics and Embodied AI Group, DTU Electro, Technical University of Denmark},%Department and Organization
            city={Kgs. Lyngby},
            postcode={DK 2800}, 
            country={Denmark}}

\author[2]{Roberto Galeazzi}[
orcid=0000-0002-6047-7922
]
\ead{roga@dtu.dk}

\cortext[1]{Corresponding author}

\begin{abstract}% no longer than 200 words
The paper presents a novel learning-based sampling strategy that guarantees rejection-free sampling of the free space under both biased and approximately uniform conditions, leveraging multivariate kernel densities. Historical data from a given autonomous system is leveraged to estimate a non-parametric probabilistic description of the domain, which also describes the free space where feasible solutions of the motion planning problem are likely to be found. The tuning parameters of the kernel density estimator, the bandwidth and the kernel, are used to alter the description of the free space so that no samples can fall outside the originally defined space.The proposed method is demonstrated in two real-life case studies: An autonomous surface vessel (2D) and an autonomous drone (3D). Two planning problems are solved, showing that the proposed approximately uniform sampling scheme is capable of guaranteeing rejection-free samples of the considered workspace. Furthermore, the effectiveness of the proposed method is statistically validated using Monte Carlo simulations.
%The paper presents a novel learning-based sampling strategy that guarantees rejection-free sampling of the free space under both biased and approximately uniform conditions, leveraging multivariate kernel densities. Historical data from past states of a given autonomous system is leveraged to estimate a non-parametric probabilistic description of the domain, which in turn also describes the free space where feasible solutions of the motion planning problem are likely to be found. The tuning parameters of the kernel density estimator, the bandwidth and the kernel, are then used to alter the description of the free space so that no sampled states can fall outside the originally defined space. The proposed method is demonstrated in two real-life case studies: An autonomous surface vessel (2D) and an autonomous drone (3D). Two planning problems are solved, showing that the proposed approximately uniform sampling scheme is capable of guaranteeing rejection-free sampling of the considered workspace. Furthermore, the planning effectiveness of the proposed method is statistically validated using Monte Carlo simulations.
\end{abstract}

%Research highlights
% \begin{highlights}
% \item 
% \end{highlights}


\maketitle

%\runningpagewiselinenumbers
%\linenumbers

\section{Introduction}
The interest in adopting autonomous robotic systems is steadily increasing in several industrial sectors, especially after witnessing the positive impact that robotization has had in large companies to perform mundane, repetitive and dangerous tasks. Companies see a growth potential in the integration of cobots into the manufacturing process; however, they are also concerned about the challenges associated with frequent changes on the production line. In addition, the transport sector is making the first experiments with the introduction of autonomous vehicles for logistics and urban mobility; however, there are open questions about the flexibility of such systems in the face of changes in the operational environment. If autonomous systems had the ability to leverage past experiences accumulated through the successful execution of tasks to plan the execution of similar tasks in a partly new environment, then the barriers to their actual adoption could be lowered. Furthermore, by leveraging past experiences, autonomous systems can plan in such a way that actions are predictable to humans. For autonomous systems such as cars and marine crafts, this will be of utmost importance, as humans are interacting with these systems on both roads and oceans, respectively.

Learning-based motion planners could provide a solution to transfer these experiences. Past states of the autonomous system collected during the completion of a task or experiences of other agents could be used to hone the search for new solutions to the given motion planning problem, within the region containing the past experiences, as long as the task or objective does not change significantly. When alterations in the task or systems workspace occur, the past states could be used to drive the exploration of previously unaccounted regions of the workspace in the vicinity of that encoded within the past experience. 

Sampling-based motion planning (SBMP) is proven to conquer complex motion planning tasks, where the given robotic system is highly constrained by its dynamics and working environment. Given the probabilistic nature of the historical data and SBMP, investigating how to integrate such experiences in the sampling procedure is a natural extension.

\subsection{Related Work}
Since the introduction of Probabilistic Roadmaps (PRM) \citep{kavraki1996probabilistic,kavraki1998analysis,geraerts2004comparative} and Rapidly-exploring Random Trees (RRT) \citep{lavalle1998rapidly,kuffner2000rrt,LaValle2001}, Sampling-based Motion Planning (SBMP) has had a strong grasp on the motion and path planning space. This class of algorithms has a set of key advantages over traditional grid-based algorithms, primarily due to its ability to deal with systems of greater complexity with less computational burden. 
State-of-the-art optimal sampling-based motion planners include RRT* and PRM* ~\citep{karaman2011sampling}, Batch Informed Trees (BIT*)~\citep{gammell2015batch}, Fast Marching Trees~(FMT*) \citep{janson2015fast} and additional variants of the aforementioned algorithms~\citep{choudhury2016regionally,strub2020adaptively,strub2020advanced,klemm2015rrt,yi2015morrf,otte2016rrtx}. By uniformly sampling the state space, the aforementioned SBMPs maintain their probabilistic guarantees and are asymptotically optimal.

% -- Sampling techniques for SBMP
However, the choice of sampling technique plays an important role in the convergence speed. The primary objective of alternative sampling schemes is to increase the probability of sampling states which can improve the current solution, compared to wasting sampling effort on states which provide no value. An extensive review of sampling techniques utilised in conjunction with RRT and its variants was performed by \cite{veras2019systematic}. The authors categorised the following sampling objectives: Goal-biased, obstacle-biased, region-based, path-biased, passage-biased, search space reduction and biasing through sampling distributions. Furthermore, \cite{gammell2021asymptotically} provided an updated overview of the current state of SBMP.

The uniform sampling strategy is applicable to all problems and also guarantees that a solution, if it exists, is found. However, for certain problems, the uniform sampling strategy tends to sample states that are infeasible due to collision or system constraints. As an alternative, \cite{abbasi2017hit} investigated planning using knowledge of the free space to form a non-convex region, which was then sampled directly using a hit-and-run sampling scheme. \cite{enevoldsen2021kde} showed that uniform sampling a triangulated representation of a non-convex environment provided a significant increase in sampling value, as the obstacle-to-free space ratio increases, compared to the baseline uniform sampler. Leveraging prior knowledge of the problem is often referred to as importance sampling \citep{lindemann2005current}, which is a class of non-uniform sampling, more specifically, a priori non-uniform sampling. The key idea is that utilising knowledge of the problem will allow for a more rapid convergence to a solution. The application of importance sampling is a general topic within the motion planning literature~\citep{janson2018monte,osa2020multimodal,schmerling2016evaluating}.

%Classical "informed" algorithm
Search space reduction was popularised by \cite{gammell2014informed}, which proposed the Informed-RRT*, a method for reducing the search space as the current best found solution improves. This is achieved by uniformly sampling a $n$-dimensional hyperspheroid, which corresponds to bounding the search space by $n$-dimensional symmetric ellipses scaled by the current best path length.
This concept was further iterated by \cite{mandalika2021guided}, who proposed a system that incrementally densifies the internal states of the bounding $n$-dimensional hyperspheroid. \cite{enevoldsen2022iros} presented a generalised extension of the informed sampling strategy, which is tailored for piece-wise linear collision avoidance applications, compared to conventional path planning problems. \cite{enevoldsen2021colregs} proposed an informed uniform sampling strategy that directly encodes the maritime rules-of-the-road (known as the COLREGs) by sampling an elliptical half-annulus. \cite{kunz2016hierarchical} proposed an alternative informed sampling scheme for kinodynamic planning, since the elliptical informed subset is not suitable for systems with kinodynamic constraints. The authors propose a hierarchical rejection sampler which can sample the relevant informed set without explicitly parameterizing it. \cite{yi2018generalizing} proposed sampling routines for generating samples in a generalised informed set, primarily using Markov Chain Monte Carlo. The authors show asymptotic optimality for generally shaped informed spaces. 

%----

Recently, a growing interest in exploring learning methods to drive the sampling strategy has emerged. \cite{iversen2016kernel} proposed a self-learning sampling scheme, where an initial uniform distribution is updated with experiences from previous paths. The update occurs by augmenting the uniform distribution with the new data and using kernel density estimation to bias the uniform search space. \cite{lehner2017repetition,lehner2018repetition} propose a similar idea, however, using Gaussian Mixture Models (GMMs) to bias the sampling towards regions with previous solutions. \cite{chamzas2019using} created a bank of local samplers by decomposing the work space into smaller regions; these local samplers are then customised to a specific problem based on prior knowledge of the task. The authors argue that the data for local samplers have high complexity and therefore represent the local space with GMMs. \cite{baldwin2010non} proposed using a variant of GMM called Infinite GMMs, which allows the authors to learn typical GMM tuning parameters based on the expert data. \cite{dong2020knowledge} demonstrated the use of a GMM sampling bias routine for automated parking, where the underlying data is generated from past parking scenarios carried out by experts. \cite{zhang2018learning} proposed modelling a rejection sampling technique using Markov Decision Processes, so that an offline policy can be modelled for environments that are similar. 

For probabilistic roadmaps, \cite{hsu1997path} propose using expansive configuration spaces, where the algorithm attempts to only sample areas of the configuration space of most relevance to the query at hand. \cite{kim2018dancing} combines sampling-based and optimisation-based planning, while approximating the configuration space. \cite{bialkowski2013free} utilises estimates of observed samples from the obstacle-free space to generate new samples. 

A two-stage approach is proposed by \cite{arslan2015machine}, where each sample is classified as collision-free or not. Kernel density estimation is used to create a collection of collision free samples, which then generates new samples with a greater probability of also being collision free. The second stage evaluates the potential of the newly generated sample, to determine whether or not it is capable of providing value. \cite{joshi2019non} further iterates on this by restricting the search space to the $L^2$ informed subset (as proposed by \cite{gammell2014informed}), while leveraging the information captured by the classifier. 

\cite{ichter2018learning} presented a conditional variational autoencoder, which is trained on past robot experiences. Non-uniform samples biased towards the past experiences are generated using the latent layer, narrowing the search for new paths to the area previously explored. Non-uniform and uniform sampling schemes are combined to retain the optimality guarantees of the given SBMP algorithm. The work was later extended into multiple networks capable of solving the entire SBMP problem~\citep{ichter2019robot}. \cite{wang2020neural} proposed Neural RRT*, where a Convolutional Neural Network (CNN) is trained with previously successful paths, and then used to generate biased samples in the neighbourhood of the provided input data. \cite{li2021mpc} proposes an imitation learning-based kinodynamic motion planner, where deep neural networks are combined with Model Predictive Control. The method is trained on historical data, to compute paths that directly adhere to the kinodynamic constraints of the given system.

\subsection{Novelty and Contribution}                     
This paper proposes a novel learning-based sampling strategy for motion planning of autonomous systems. The method takes advantage of past experiences from prior motions to efficiently find new solutions to the motion planning problem in the presence of changes \textcolor{\revOne}{to the static obstacle configuration within the} workspace. To achieve this, kernel density estimation with a finite support kernel is adopted to generate a non-parametric probabilistic description of regions of the free space where feasible solutions of the motion planning problem are likely to be found, in the neighbourhood of past data. The bandwidth of the multivariate probability density function is exploited to redefine the free space by enlarging obstacle regions. \textcolor{\revOne}{Sampling the KDE in this new restricted space, with the modified obstacle regions, ensures that the generated samples always fall within the originally free space}. The paper also shows how the estimated kernel density can be exploited to obtain weights for performing importance sampling in the neighbourhood of past motions, for both biased and approximately uniform sampling of the free space, allowing the motion planner to both improve the current solution, but also explore nearby regions to the estimated one in relation to a new planning problem. Both sampling strategies are guaranteed to be rejection-free by construction. \textcolor{\revOne}{Exploiting the kernel density estimate of the free space to generate guaranteed rejection-free approximately uniform samples within the domain of the historical data, is what sets the proposed sampling scheme apart from conventional methods within learning-based samplers and SBMPs. The existing work tends towards biased sampling, in order to generate samples that allow the SBMP to converge to the neighbourhood surrounding past solutions, whereas the proposed method, in addition to biased sampling, emphasises exploration of the entire data, in an approximately uniform and guaranteed rejection-free manner.}

The presented sampling strategy is verified in two case studies that address motion planning for an autonomous ship sailing in coastal waters and for an aerial drone performing a complex inspection task in a confined space. \textcolor{\revOne}{The proposed method is compared to a baseline uniform sampling algorithm, as both focus on uniformity, rather than biased sampling.}

\section{Preliminaries}
The following sections provide some preliminary and foundational details regarding the workspace formulation, the SBMP problem and baseline sampling strategies.
\subsection{The Workspace}
Consider some autonomous (robotic) system, for example, an unmanned aerial vehicle (UAV) or an autonomous surface vehicle (ASV), that operates in some workspace $\mathcal{W}$ that is a subset of the Euclidean space $\mathbb{R}^n$.
Let $\mathcal{A}$ be some instance of the aforementioned autonomous system, then $\mathbf{x}_i$ is an instance of coordinates that determines the current state of the system. 
Obstacles present in the workspace of the autonomous system can also be mapped in the space $\mathcal{W}$. Let $\mathcal{O}$ be an obstacle in the workspace $\mathcal{W}$, then the obstacle in the workspace $\mathcal{W}_{\mathrm{obs}} \subseteq \mathcal{W}$, is defined as $\mathcal{W}_{\mathrm{obs}} \triangleq\left\{\mathbf{x}\in\mathcal{W} | \mathcal{A}(\mathbf{x}) \cap \mathcal{O} \ne \emptyset\right\}$, i.e. the set of all states in which the autonomous system collides with the obstacle. The complementary set of the obstacle space is called the free space, that is $\mathcal{W}_{\mathrm{free}} = \mathcal{W} \setminus \mathcal{W}_{\mathrm{obs}}$.

\subsection{Sampling-based Motion Planning}
\textcolor{\revOne}{The intended use of the proposed method is for computing low-cost feasible solutions in connection with the solution of the optimal sampling-based motion planning problem, defined similarly to \cite{gammell2014informed}.} 

Consider the state space $\mathcal{X}$, consisting of two subsets, namely $\mathcal{X}_{\text{free}}$ and $\mathcal{X}_{\text {obs}}$, with $\mathcal{X}_{\text{free}}=\mathcal{X} \backslash \mathcal{X}_{\text{obs}}$. The space $\mathcal{X}_{\text{free}}$ contains all states that are feasible with respect to the given system and its operating environment. Let $\mathbf{x}_{\text{start}} \in \mathcal{X}_{\text{free}}$ be the initial state at the initial time $t=t_0$ and $\mathbf{x}_{\text{goal}} \in \mathcal{X}_{\text{free}}$ be the desired final state. 

Let $\sigma:[0,1] \mapsto \mathcal{X}$ be a sequence of states that constitute a found path, and $\Sigma$ be the set of all feasible and non-trivial paths. The objective is to compute the optimal path $\sigma^{*}$, which minimises a cost function $c(\cdot)$ while connecting $\mathbf{x}_{\text{start}}$ to $\mathbf{x}_{\text{goal}}$ through $\mathcal{X}_{\text{free}}$, i.e.,  
\begin{equation}
	\begin{split}
		\sigma^{*}=\underset{\sigma \in \Sigma}{\arg \min }\left\{c(\sigma) \mid \right. \sigma(0)=\mathbf{x}_{\text {start }},\, &\sigma(1)=\mathbf{x}_{\text {goal }}, \\
		\forall s \in[0,1],\, &\sigma(s) \in \left. \mathcal{X}_{\text {free }}\right\}.
	\end{split}
\end{equation}
For the remainder of the paper, it is assumed that the above-mentioned state space quantities are equal to their workspace counterparts, i.e. $\mathcal{W} = \mathcal{X}$, $\mathcal{W}_{\text{free}} = \mathcal{X}_{\text{free}}$ and $\mathcal{W}_{\text{obs}} = \mathcal{X}_{\text{obs}}$.

\subsection{Uniform Sampling Strategies}
Several sampling techniques exist to obtain new nodes for the exploration of a given state space. \emph{Uniform random sampling} is the simplest strategy to achieve uniform exploration of the space and is based on the random selection of values for each degree of freedom present in the state $x\in\mathcal{X}$. Deterministic methods also exist, where sampling is driven by a \emph{low-dispersion objective} or a \emph{low-discrepancy objective}. The former leads to the use of a grid whose resolution changes so that samples are placed to minimise the size of the uncovered areas. The latter addresses the shortcomings that arise from having grids that are aligned with the coordinate axes of the space. Among these sampling strategies, we find the Halton sequence (and its variants) and lattices.

\begin{figure*}[htb]
	\centering
	\begin{subfigure}[b]{0.95\columnwidth}
		\centering
		\includegraphics[width=\textwidth]{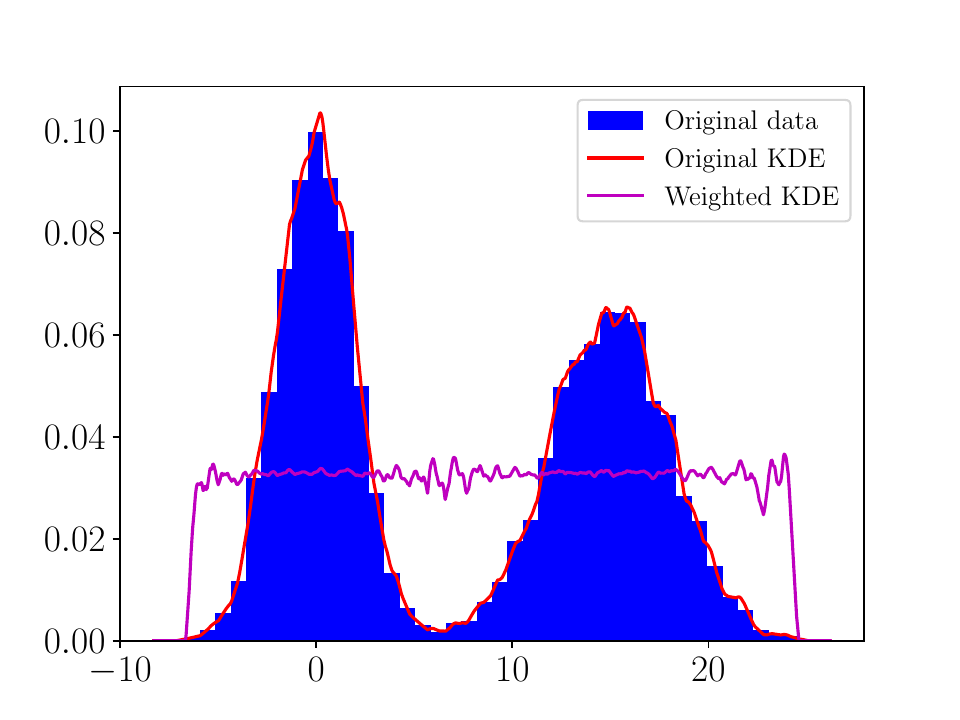}
		\caption{Resulting weighted KDE}
		\label{fig:kde_reweight_1d_toy}
	\end{subfigure}
	\begin{subfigure}[b]{0.99\columnwidth}
		\centering
		\includegraphics[width=\textwidth]{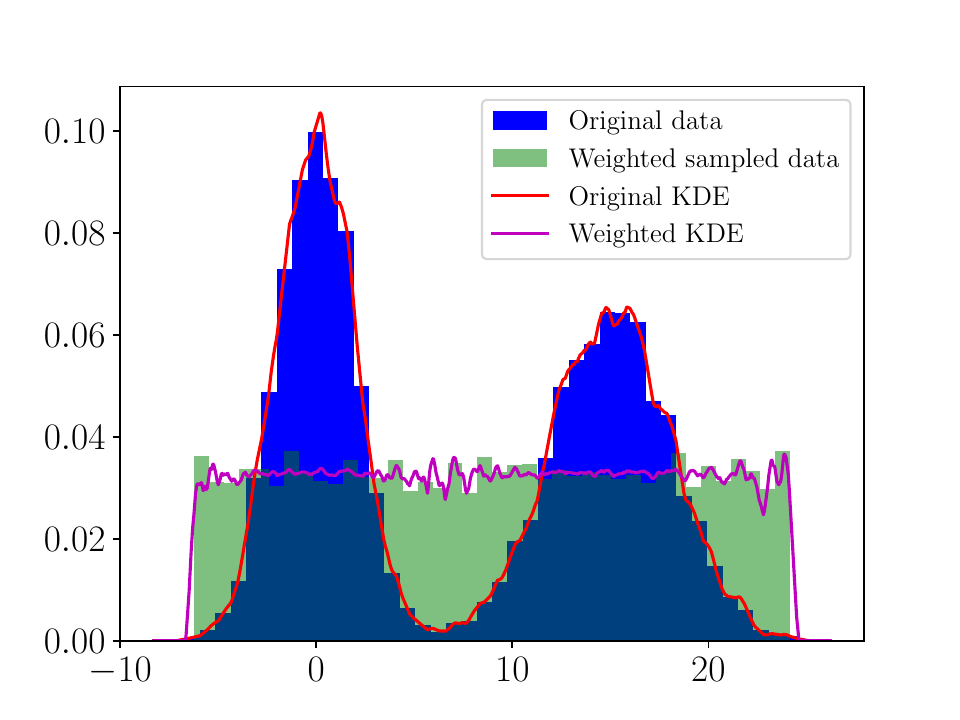}
		\caption{Approximate uniform sampling}
		\label{fig:kde_reweight_resample_1d_toy}
	\end{subfigure}
	\caption{One dimensional resampling of the weighted KDE, such that uniform samples of the KDE domain are generated.}\label{fig:kde_1d_toy}
\end{figure*}

\section{Guaranteed Rejection-free Sampling of Non-parametric Spaces}\label{sec:propsed_method}
Consider the autonomous system $\mathcal{A}$ performing tasks in the free space $\mathcal{W}_{\mathrm{free}}$ over an arbitrary large time horizon $T$. In the event of changes of the free space due to e.g., the introduction of new obstacles or operational boundaries, it is of interest to produce new motion plans by leveraging the information carried by the historical paths $\sigma_i$ traversed by $\mathcal{A}$ over the period $T$. We propose to leverage the available data from past experiences in order to compute a non-parametric probabilistic representation of the free space that describes the space in which new solutions may exist. Such a non-parametric description can be utilized twofold. Directly sampling the non-parametric distribution provides a biased sampling strategy, aiding in the computation of motion plans similar to those encoded within the data. The non-parametric distribution can also be used to approximately uniformly explore nearby regions of the free space to that traversed by the paths $\sigma_i$, yet remaining within the defined free space.

\subsection{Kernel Density Estimation}
Multivariate kernel density estimation (KDE) is a non-parametric method to estimate an unknown multivariate probability density function $f(\mathbf{x})$ based on a finite data set containing realisations of the multivariate random variable described by $f(\mathbf{x})$. Specifically, the kernel density estimator operates on a set of $n$ data vectors each of which is an identically distributed $p$-variate random variable drawn from the same and unknown distribution $f(\mathbf{x})$. 

Let $X = \left\{\mathbf{x}_i \in \mathbb{R}^p | \mathbf{x}_i \sim f, \, i= 1,\ldots, n\right\}$ be the available data set; then the general form of the $p$-variate kernel density estimator is given by~\cite{gramacki2018nonparametric} 
\begin{equation}
	\hat{f}_{X}(\mathbf{x},\mathbf{H}) = \frac{1}{n}\sum_{i = 1}^n|\mathbf{H}|^{-1/2}K(\mathbf{H}^{-1/2}(\mathbf{x} - \mathbf{x}_i))
\end{equation}
where $\mathbf{x}_i = [x_{i1}, x_{i2}, \ldots, x_{ip}]^\top \in X$, $\mathbf{x} = [x_1, x_2, \ldots, x_p]^\top \in \mathbb{R}^p$ is an arbitrary element, $\mathbf{H} = \mathbf{H}^\top > 0$ is the non-random $p\times p$ bandwidth matrix, and $K(\cdot)$ is the kernel function. 

The kernel and the bandwidth are the tuning parameters; however, as the size of the data set increases, the importance of choosing the bandwidth outweighs the particular choice of the kernel~\cite[Section 6.2.3]{scott2015multivariate}.
In the case of 1D KDEs, if the underlying distribution is unimodal or exhibits normal features, then Silverman's rule or Scott's rule \citep{silverman1982algorithm} can be applied to compute the bandwidth. For 1D instances where the data has multiple modes, the Improved Sheather-Jones algorithm serves as a plug-in bandwidth selector \citep{botev2010kernel}. When multivariate kernel density estimation is performed, such rule-based approaches do not apply. However, data-driven methods exist to compute optimal kernel functions and bandwidth matrices, as shown in \cite{o2016fast,gramacki2017fft}.

\subsection{Generating Samples From the Estimator \lowercase{$\hat{f}$}}\label{sec:samplekde}
The generation of samples from parametric distributions is achieved through inverse transform sampling, in order to uniformly create samples belonging to the corresponding probability density function. Performing such an inversion of a KDE poses several challenges. However, by construction, a KDE $\hat{f}_{X}(\mathbf{x},\mathbf{H})$ consists of a mixture of the kernel function $K(\cdot)$. This means that the KDE itself can be reconstructed by sampling the data used to generate it, biased by the chosen kernel function and its parameters.

One can generate $m$ samples from the estimated PDF $\hat{f}_{X}(\mathbf{x},\mathbf{H})$, given the data set $X = \left\{\mathbf{x}_{1}, \mathbf{x}_{2},\dots, \mathbf{x}_{n}\right\}$ used to compute the estimate $\hat{f}$, generate $m$ indices $(k_1, k_2, \dots, k_m)$ from the discrete uniform density $U(1,2,\dots, n)$ in order to uniformly select data from the original data set.  Each selected point is then biased by a sample $\mathbf{t}_i$ generated from the chosen scaled multivariate kernel~\citep{devroye1985nonparametric, scott2015multivariate}
\begin{equation}\label{eq:generate_sample}
	\mathbf{s}_i = \mathbf{x}_{k_i} + \mathbf{t}_i, \quad i = 1, \dots, m .
\end{equation}
This procedure naturally leads to generating samples that are biased towards regions of higher density, and as the number of sampled points $m$ tends toward infinity, the samples represent a densified estimate of the non-parametric distribution of past states. 

For sampling-based motion planning, it may instead be desired that the generated samples are uniformly distributed over the domain of the KDE. This can be achieved by computing the densities for each data point in the set $X$,
\begin{equation}\label{eq:extract_weights}
	\omega_i = \hat{f}_{X}(\mathbf{x}_i,\mathbf{H}), \quad i = 1, \ldots, j
\end{equation}
and using the reciprocal $1/\omega_i$ to weight the selection of the indices $(k_1, k_2, \dots, k_m)$. Fig.~\ref{fig:kde_1d_toy} shows a one-dimensional implementation of such procedure, where Fig.~\ref{fig:kde_reweight_1d_toy} illustrates the estimated KDE and the computation of the weights, while Fig.~\ref{fig:kde_reweight_resample_1d_toy} shows the reweighted sampling of the KDE domain compared to the original KDE.
\begin{rem}
	There are complexities in generating the kernel sample $\mathbf{t}_i$, depending on the chosen kernel function. The most common case is a KDE using the Gaussian kernel, where sampling occurs by selecting the data points as described above and subsequently biasing each sample by a zero-mean $p$-variate normal distribution with the covariance described in terms of the bandwidth.
\end{rem}
\begin{rem}
	To ensure that the given sampling-based motion planner maintains its asymptotic optimality, the proposed method should be combined with uniform sampling of the entire space~\citep{ichter2018learning}. This results in $\lambda m$ samples drawn from the proposed scheme and $(1 - \lambda)m$ samples from a uniform sampler, where $0 < \lambda < 1$ is a tuning parameter based on the available data set and the problem at hand.
\end{rem}

\begin{figure}
	\centering
	\includegraphics[width=0.85\columnwidth]{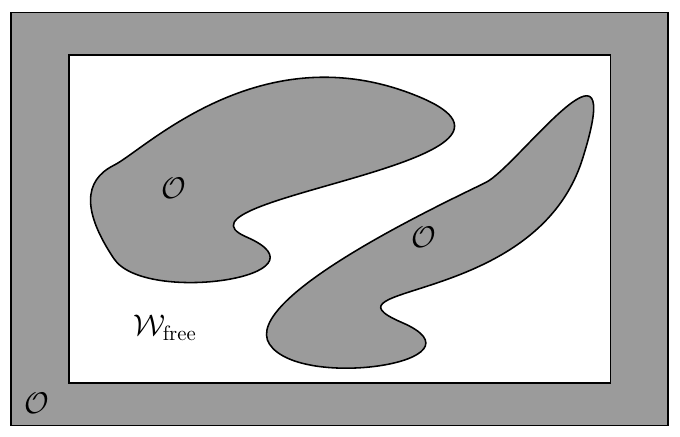}
	\caption{Example of closed and bounded free space $\mathcal{W}_\mathrm{free}$.}
	\label{fig:theorem_figure}
\end{figure}

\subsection{Guaranteed Rejection-free Sampling Scheme}\label{sec:truncKDE}
Consider the autonomous system $\mathcal{A}$ operating in space $\mathcal{W} \subseteq \mathbb{R}^p$ in the presence of a set of obstacles $\mathcal{O}_0$. Then the free space is defined as
\begin{equation}
	\mathcal{W}_{\mathrm{free}}^0 = \left\{ \mathbf{x} \in \mathcal{W} \, | \, \mathcal{A}(\mathbf{x}) \cap \mathcal{O}_0 = \emptyset \right\}. 
\end{equation}
Depending on the configuration of the obstacles, the free space $\mathcal{W}_{\mathrm{free}}^0$ can be either an open and unbounded set, or a closed and bounded. Fig.~\ref{fig:theorem_figure} shows an exemplification of the latter. The following theorem is valid in both cases.

\begin{thm}
	Let $X$ be the set of states $\mathbf{x}_i\in\mathcal{W}_\mathrm{free}^0$ that an agent $\mathcal{A}$ has assumed during a time period $T = [t_0,t_1] \in \mathbb{R}$, $t_1 > t_0 \geq 0$, and $f(\mathbf{x})$ the unknown spatial distribution of such states over the free space $\mathcal{W}_\mathrm{free}^0$. Assume that at time $t_2 > t_1$ the free space $\mathcal{W}_\mathrm{free}^0$ changes to $\mathcal{W}_\mathrm{free}^1$, where $\mathcal{W}_\mathrm{free}^1 \cap \mathcal{W}_\mathrm{free}^0 \neq \emptyset$. Then the rejection-free sampling of $\mathcal{W}_\mathrm{free}^1$ is guaranteed through the sampling of $\hat{f}_X(\mathbf{x},\mathbf{H})$ on the set $\bar{X} = \{\mathbf{x} \in \mathcal{W}_{\mathrm{free}}^2 \cap X\}\subseteq X$, where $\hat{f}_X(\mathbf{x},\mathbf{H})$ is a kernel density estimator of $f(\mathbf{x})$ with finite support kernel, and $\mathcal{W}_\mathrm{free}^2\subset\mathcal{W}_\mathrm{free}^1$, \textcolor{\revOne}{as long as the static obstacle configuration in $\mathcal{W}_\mathrm{free}^1$ remains constant whilst generating the KDE $\hat{f}_X(\mathbf{x},\mathbf{H})$, solving the SBMP, and executing the planned motion.}
\end{thm}
\begin{proof}
	Given the set $X$, the $p$-variate KDE with finite support kernel $K(\cdot)$ (e.g., the box or Epanechnikov's kernel), and bandwidth matrix $\mathbf{H}$ \cite[Section 6.2.3]{scott2015multivariate}
	\begin{equation}
		\hat{f}_{X} : \mathcal{W}_\mathrm{free}^0 \longrightarrow E
	\end{equation}
	is an estimator of the unknown spatial distribution $f(\mathbf{x})$, which maps elements of the free space to density values, where 
	\begin{equation}
		E = \{ e \in  \mathbb{R}_+ | 0 \leq e \leq 1\} .
	\end{equation}
	$\hat{f}_{X}$ is a biased non-parametric probabilistic description of $\mathcal{W}_\mathrm{free}^0$, whose sampling allows to plan the motion of the agent $\mathcal{A}$ within the free space $\mathcal{W}_\mathrm{free}^0$.
	
	At an arbitrary time instant $t_2 > t_1$ the free space partly changes such that $\mathcal{W}_\mathrm{free}^1 \cap \mathcal{W}_\mathrm{free}^0 \neq \emptyset$, i.e., 
	\begin{equation}
		\mathcal{W}_{\mathrm{free}}^1 = \left\{ \mathbf{x} \in \mathcal{W} \, | \, \mathcal{A}(\mathbf{x}) \cap \mathcal{O}_1 = \emptyset \right\}
	\end{equation}
	where $\mathcal{O}_1$ is the new set of obstacles. 
    
	Let $\mathcal{B} = \left\{\mathbf{x} \in \mathbb{R}^p \, | \, \|\mathbf{x}\|\leq \varrho \right\}$ be a ball of radius $\varrho$, where $\varrho = (\lambda_{\min}(\mu_2(K)\mathbf{I}_p))^{-1} \|\mathbf{H}^{1/2}\|_q$, $\mu_2(K)\mathbf{I}_p$ is the second order moment of the selected kernel function $K(\cdot)$ \cite[Section 3.6]{hardle2004nonparametric}, and $\lambda_{\min}$ is its smallest eigenvalue. The new free space $\mathcal{W}_{\mathrm{free}}^2$ is defined by enlarging the obstacle regions $\mathcal{O}_1$, i.e.,  
	\begin{align}
		\mathcal{O}_2 &= \mathcal{O}_1 \oplus \mathcal{B} \\
		\mathcal{W}_{\mathrm{free}}^2 &= \left\{ \mathbf{x} \in \mathcal{W} \, | \, \mathcal{A}(\mathbf{x}) \cap \mathcal{O}_2 = \emptyset \right\}.
	\end{align}
	and by construction $\mathcal{W}_{\mathrm{free}}^2 \subset \mathcal{W}_{\mathrm{free}}^1$.
	
	The sampling of $\hat{f}_{X}(\mathbf{x},\mathbf{H})$ on the set $\bar{X} = X \cap \mathcal{W}_{\mathrm{free}}^2$ achieves the rejection-free sampling of $\mathcal{W}_{\mathrm{free}}^1$. To guarantee that all samples fall within $\mathcal{W}_{\mathrm{free}}^1$ the chosen kernel $K(\cdot)$ must have finite support, otherwise the use of the kernel $K(\cdot)$ on data points belonging to $\bar{X}$ could generate samples falling outside $\mathcal{W}_{\mathrm{free}}^1$.
\end{proof}
\textcolor{\revOne}{By generating samples as outlined in Section \ref{sec:samplekde}, based on the reduced dataset $\bar{X}$, one guarantees that the generated samples lie within $\mathcal{W}^1_{\mathrm{free}}$, as the reduced dataset by design falls within $\mathcal{W}^2_{\mathrm{free}}$. By generating samples, that is, applying bias to sampled data from $\bar{X}$, using the chosen kernel function, the newly generated samples are at most the distance of the bandwidth from the data in $\bar{X}$. Since the obstacles have been inflated by the bandwidth when generating $\mathcal{W}^2_{\mathrm{free}}$, it then ensures that if a given data point within $\bar{X}$ lies on the boundary of $\mathcal{W}^2_{\mathrm{free}}$, the resulting generated sample will at its extreme be on the boundary of $\mathcal{W}^1_{\mathrm{free}}$, thereby ensuring that all samples generated by $\hat{f}_{X}(\mathbf{x},\mathbf{H})$ fall within $\mathcal{W}^1_{\mathrm{free}}$. The requirement for the finite support kernel function comes from the link between the kernel bandwidth and obstacle inflation, as it is crucial for ensuring that the generated samples are guaranteed to be within $\mathcal{W}^1_{\mathrm{free}}$ in which the planning problem takes place.}
\begin{figure*}
	\centering
	\begin{subfigure}[b]{0.51\columnwidth}
		\centering
		\includegraphics[width=\textwidth]{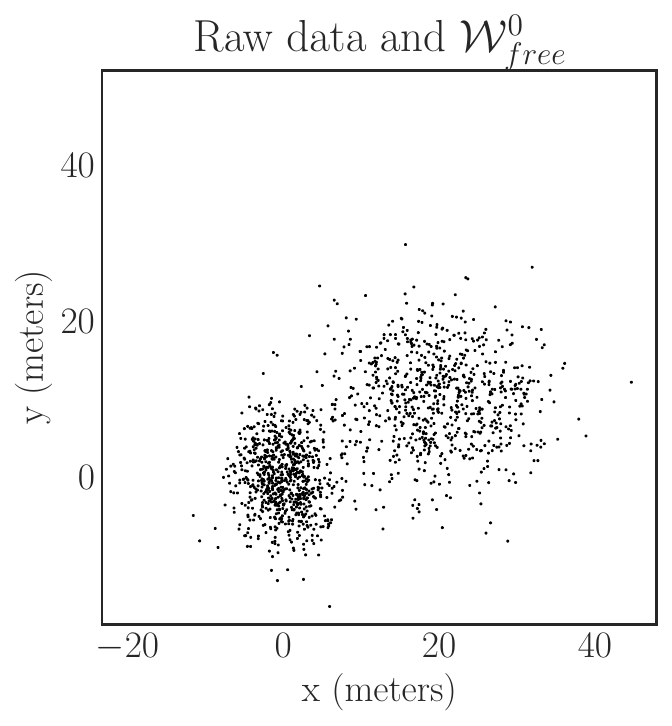}
		\caption{ $X$ and $\mathcal{W}^0_{\text{free}}$ }
		\label{fig:toy_raw_data}
	\end{subfigure}
	\begin{subfigure}[b]{0.51\columnwidth}
		\centering
		\includegraphics[width=\textwidth]{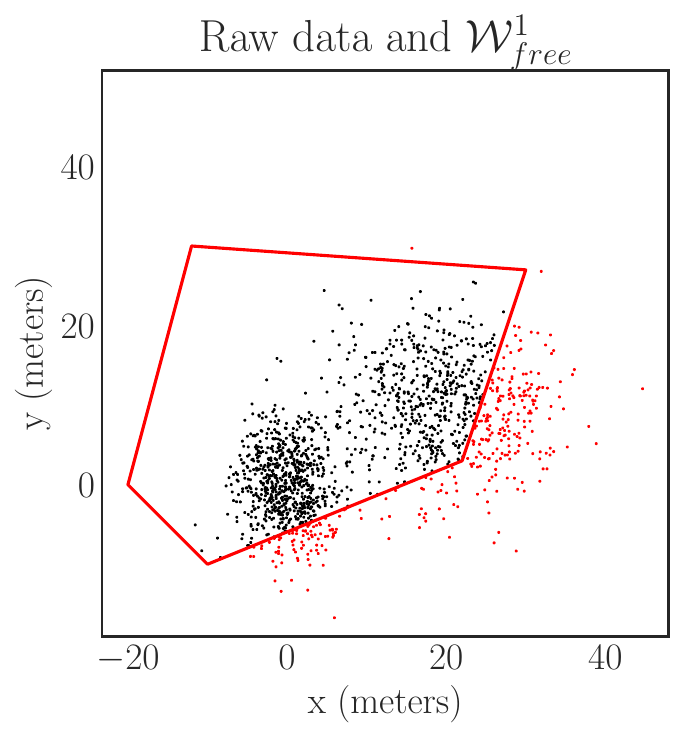}
		\caption{Imposing $\mathcal{W}^1_{\text{free}}$ on $X$}
		\label{fig:toy_raw_data_w1}
	\end{subfigure}
	\begin{subfigure}[b]{0.51\columnwidth}
		\centering
		\includegraphics[width=\textwidth]{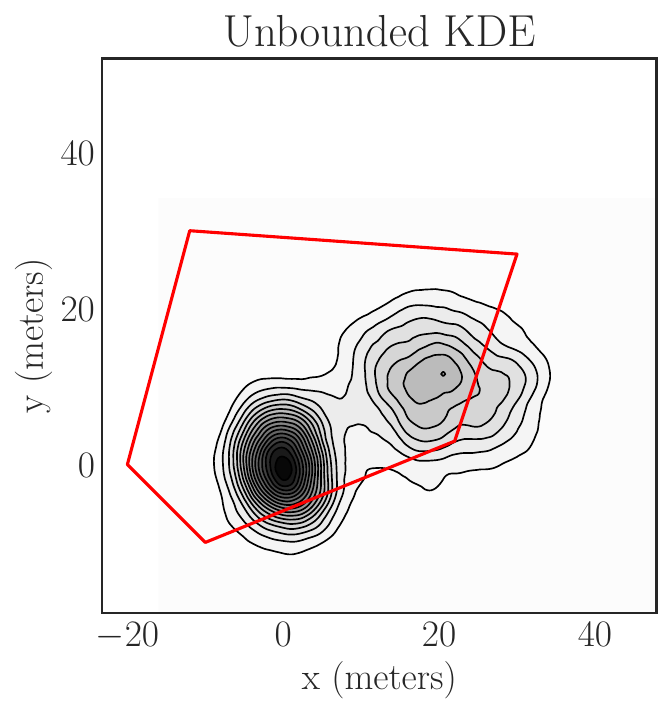}
		\caption{KDE based on data $X$}
		\label{fig:toy_unbounded_kde}
	\end{subfigure}
	\begin{subfigure}[b]{0.51\columnwidth}
		\centering
		\includegraphics[width=\textwidth]{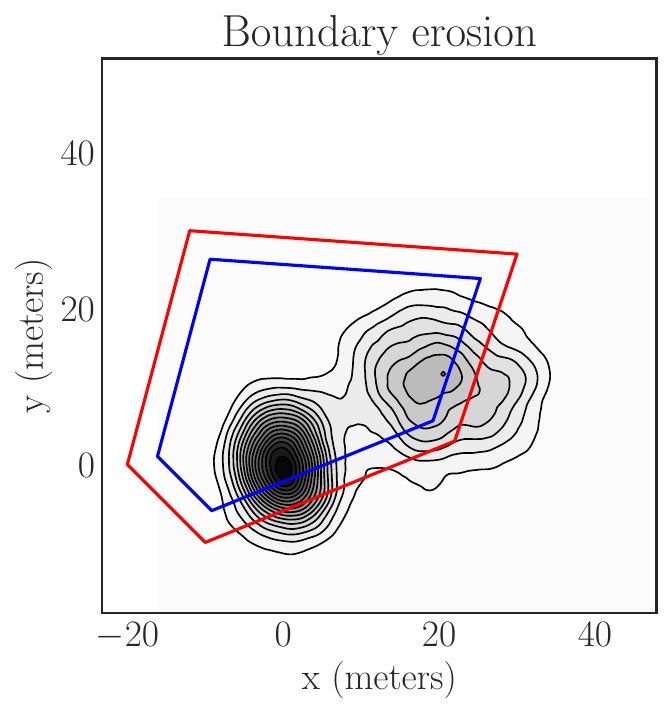}
		\caption{Erosion to create $\mathcal{W}^2_{\text{free}}$}
		\label{fig:toy_erosion}
	\end{subfigure}
	\begin{subfigure}[b]{0.51\columnwidth}
		\centering
		\includegraphics[width=\textwidth]{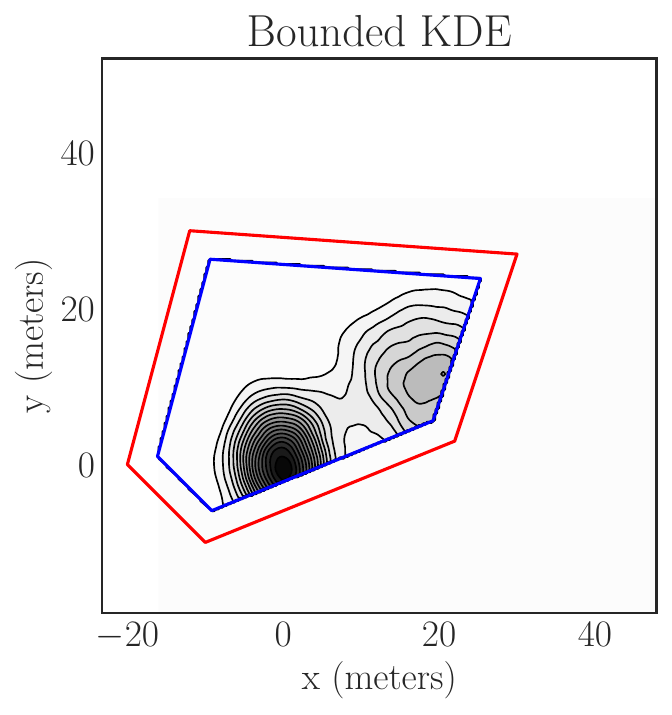}
		\caption{Truncation of KDE}
		\label{fig:toy_bounded_kde}
	\end{subfigure}
	\begin{subfigure}[b]{0.51\columnwidth}
		\centering
		\includegraphics[width=\textwidth]{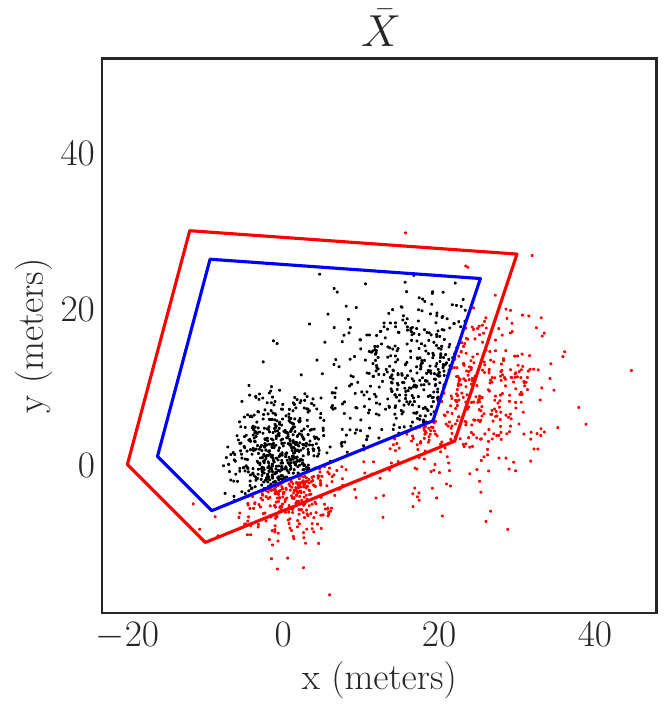}
		\caption{Imposing $\mathcal{W}^2_{\text{free}}$ on $X$ ($\bar{X}$)}
		\label{fig:toy_removing_samples}
	\end{subfigure}
	\begin{subfigure}[b]{0.51\columnwidth}
		\centering
		\includegraphics[width=\textwidth]{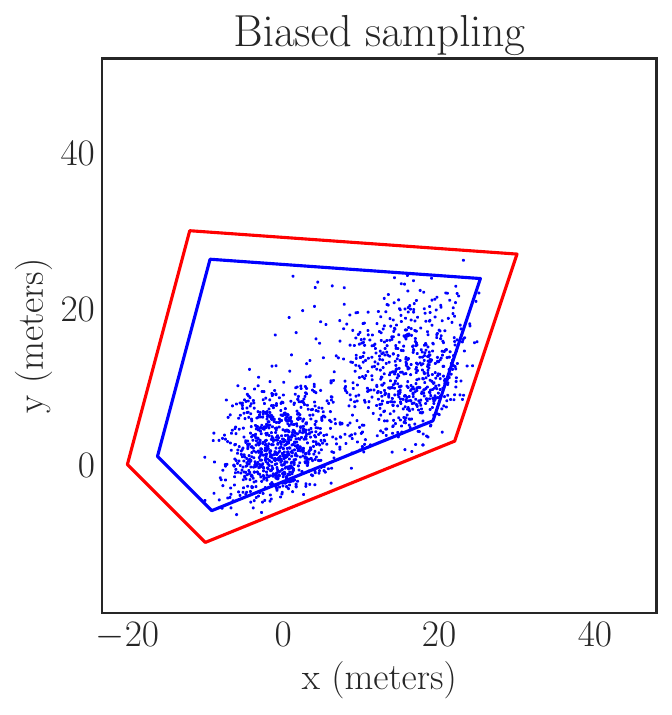}
		\caption{Sampling $\hat{f}_{\bar{X}}$ directly}
		\label{fig:toy_biased_sampling}
	\end{subfigure}
	\begin{subfigure}[b]{0.51\columnwidth}
		\centering
		\includegraphics[width=\textwidth]{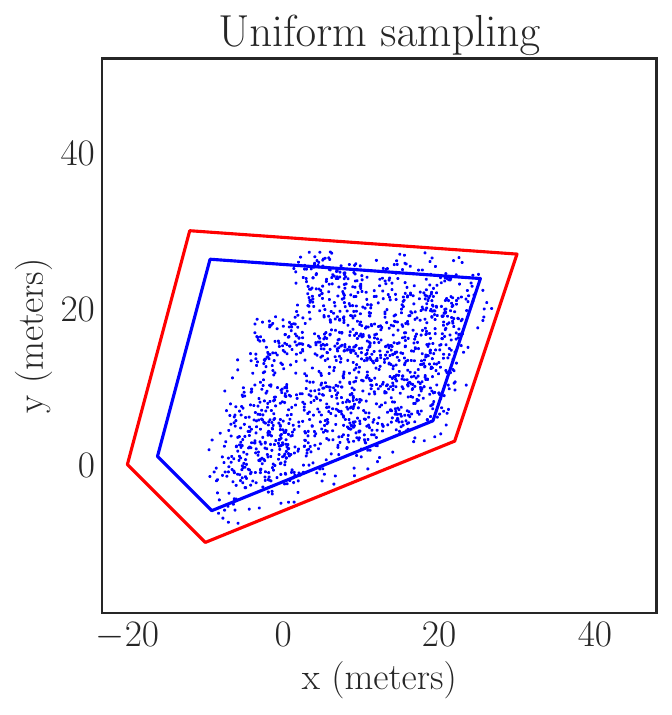}
		\caption{Sampling weighted $\hat{f}_{\bar{X}}$}
		\label{fig:toy_uniform_sampling}
	\end{subfigure}
	\caption{Visual representation of each step of the proposed method, as described in Section \ref{sec:propsed_method}. Once the space $\mathcal{W}_{\text{free}}^2$ (inner blue polygon) has been generated based on the selected bandwidth and the corresponding bounded KDE, one guarantees the ability to sample the $\mathcal{W}_{\text{free}}^1$ space without rejection sampling, within the domain covered by the historical data. Fig.~\ref{fig:toy_biased_sampling} and Fig.~\ref{fig:toy_uniform_sampling} demonstrates the ability to generate rejection-free samples in both a biased and approximately uniform manner.} 
	\label{fig:toy_example}
\end{figure*}
\begin{rem}
	To obtain the approximately uniform sampling of $(\bar{X} \oplus \mathcal{B}) \subset \mathcal{W}_{\mathrm{free}}^1$ the $p$-variate KDE $\hat{f}_{X}$ should be reweighted, \textcolor{\revOne}{as the volume of the estimated PDF is no longer equal to 1}. This can be achieved through two approaches: (i) the estimation of a new KDE based on the data set $\bar{X}$; (ii) the truncation and normalization of the original KDE. Following the latter, $\hat{f}_{X}(\mathbf{x},\mathbf{H})$ is first truncated by zeroing the densities that falls outside $\mathcal{W}_{\mathrm{free}}^2$, and then normalized to ensure that the resulting function still qualifies as a density, i.e.,
	\begin{align}
		\bar{f}_{\bar{X}}(\mathbf{x},\mathbf{H}) &= 
		\begin{cases}
			\hat{f}_{X}(\mathbf{x},\mathbf{H}),&\forall \mathbf{x} \in \bar{X} = X \cap \mathcal{W}_{\mathrm{free}}^2\\
			0,              & \mathrm{otherwise}
		\end{cases} \\
		\hat{f}_{\bar{X}}(\mathbf{x},\mathbf{H}) &= \frac{\bar{f}_{\bar{X}}(\mathbf{x},\mathbf{H})}{\rho}
	\end{align}
	where
	\begin{equation}
		\rho = \int_{\mathbb{R}^p} \bar{f}_{\bar{X}}(\mathbf{s},\mathbf{H}) \, \mathrm{d}\,\mathbf{s} .
	\end{equation}
\end{rem}
\textcolor{\revOne}{Any collected data from past instances of obstacle configurations in $\mathcal{W}^0_{\mathrm{free}}$ can be used to calculate a KDE based on the updated workspace $\mathcal{W}^1_{\mathrm{free}}$. It should be noted that the guaranteed rejection-free sampling scheme generated for a given $\mathcal{W}^1_{\mathrm{free}}$ only guarantees rejection-free sampling for the static obstacle configuration presented in that particular workspace.}

\subsection{Toy Example}
The following section presents a toy example to provide a detailed demonstration of how the proposed method is applied to a given problem. 

Using two bivariate normally distributed random variables $X_1 \sim \mathcal{N}(\mu_1, \Sigma_1)$ and $X_2 \sim \mathcal{N}(\mu_2, \Sigma_2)$ with parameters $\mu_1 = \mathbf{0}$, $\Sigma_1 = \text{diag}([10, 20])$, $\mu_2= \begin{bmatrix}20 & 10\end{bmatrix}^T$ and $\Sigma_2 = \text{diag}([45,35])$, some historical data, in this case 1500 samples, is generated, such that a collection of data points 
\begin{equation}
	X = [(x_1, y_1), (x_2, y_2), \dots, (x_n, y_n)]^\top    
\end{equation}
is created. This particular data set is created such that it represents past states that at some point were feasible with respect to the nominal space $\mathcal{W}_{\mathrm{free}}^0$. However, at the current time instant, a boundary limitation has been imposed described by the following polygon $P$
\begin{equation*}
	P = [(-10, -10), (22, 3), (30, 27), (-12, 30), (-20, 0)]
\end{equation*}
which now represents the new free space $\mathcal{W}_{\mathrm{free}}^1$. Fig.~\ref{fig:toy_raw_data} visualises the historical data and nominal space, whereas Fig.~\ref{fig:toy_raw_data_w1} visualises the imposed boundary $\mathcal{W}_{\mathrm{free}}^1$ and its impact on the historical data. 
An unbounded KDE is computed based on the data set $X$ using the bandwidth matrix $\mathbf{H} = 2\mathbf{I}$ and the finite support Epanechnikov kernel. Fig.~\ref{fig:toy_unbounded_kde} shows a visualisation of such KDE. 
To guarantee that sampling of the new free space $\mathcal{W}_{\mathrm{free}}^1$ occurs completely rejection-free, a final space is introduced. This space is an erosion (or shrinkage) of the polygonal boundary (for general obstacle regions inflation is instead required), resulting in the space $\mathcal{W}_{\mathrm{free}}^2$, as shown in Fig.~\ref{fig:toy_erosion}. The amount of erosion is related to the bandwidth matrix $\mathbf{H}$ through the radius $\varrho$. 
\begin{figure*}
	\centering
	\begin{subfigure}[b]{0.64\columnwidth}
		\centering
		\includegraphics[width=\textwidth]{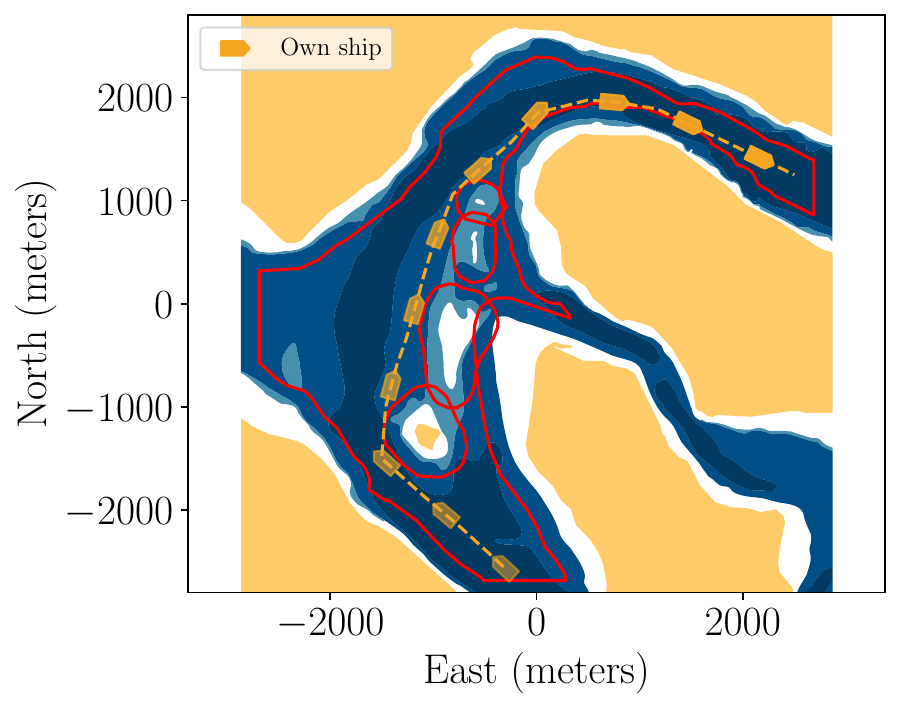}
		\caption{ASV scenario}
		\label{fig:scenario_ASV}
	\end{subfigure}
	\begin{subfigure}[b]{0.64\columnwidth}
		\centering
		\includegraphics[width=\textwidth]{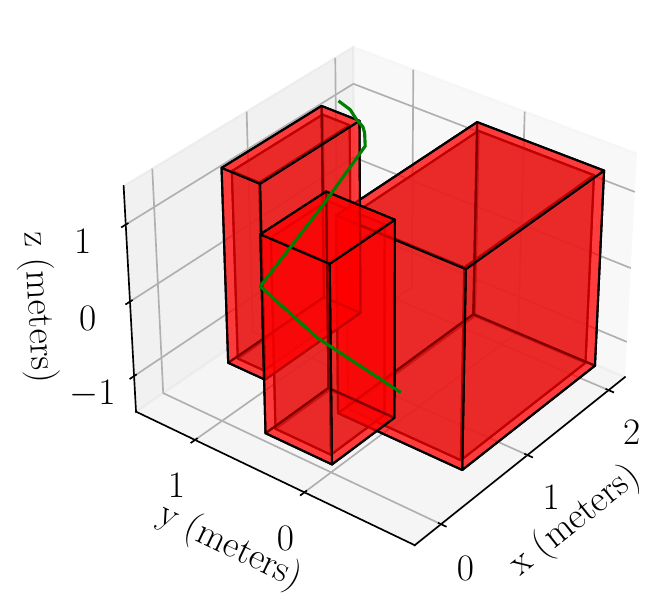}
		\caption{Drone scenario}
		\label{fig:drone_scenario_path}
	\end{subfigure}
	\begin{subfigure}[b]{0.64\columnwidth}
		\centering
		\includegraphics[width=0.85\textwidth]{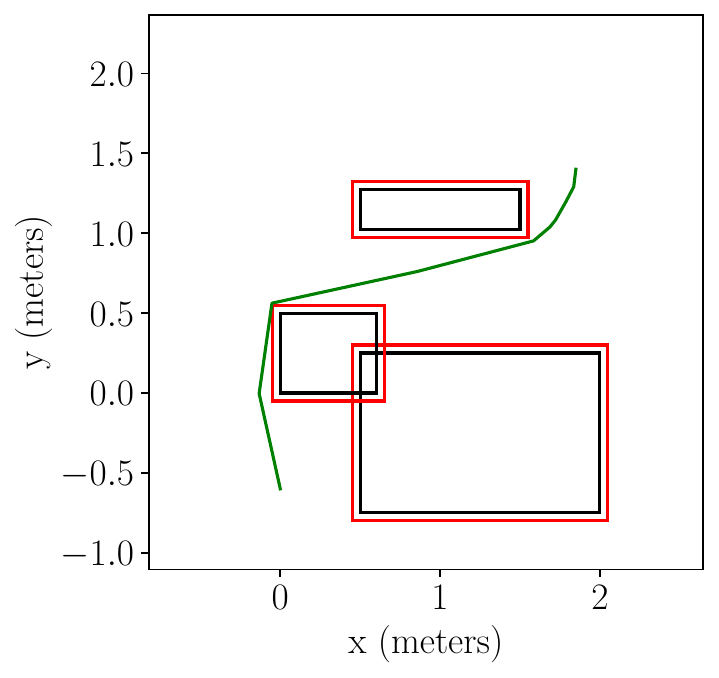}
		\caption{2D projection of Fig.~\ref{fig:drone_scenario_path}}
		\label{fig:drone_scenario_path_2d}
	\end{subfigure}
	\caption{Outcomes from the selected case studies. Fig.~\ref{fig:scenario_ASV} shows a path where the vessel maintains the specified safety distance to shallow waters. The proposed method ensures that only safe samples are generated, thereby increasing the efficiency of computing low cost solutions, as evident in Table \ref{tab:stopping_costs_ASV} and Fig.~\ref{fig:acces_box_plots}. Details regarding safe sampling-based motion planning for ASVs can be found in \cite{enevoldsen2021kde} and \cite{enevoldsen2022oe}. Fig.~\ref{fig:drone_scenario_path} details a similar scenario, but instead for the drone. This particular case study mimics an inspection task, and therefore the drone must also maintain some safety distance from the obstacles (see \cite{inspectdrone_data} for more details). Note that due to limitations with the 3D engine used for plotting, the location of the data points within the figure may be deceptive, see therefore instead the 2D projection in Fig.~\ref{fig:drone_scenario_path_2d}.}
	\label{fig:case_studies}
\end{figure*}
\begin{table}
	\caption{Computing obstacle-free samples (from $\mathcal{W}^1_{\text{free}}$) 10,000 times (Fig.~\ref{fig:scenario_ASV}). Comparison between the baseline (B) and proposed approach (R), where the percentage ($\Delta \%$) is computed as $(R-B)/B$, where lower numbers are better. Mean ($\mu$), median $\bar{\mu}$ and standard deviation ($\sigma$).}
	\label{tab:ASV_sampling_speeds}
	\begin{center}
		{\renewcommand{\arraystretch}{1.15}% for the vertical padding
			\resizebox{\columnwidth}{!}{
				\begin{tabular}{rrrrcrrrr}
					\toprule
					& \multicolumn{3}{c}{Samples} & \phantom{a}&  \multicolumn{3}{c}{Times}\\
					\cmidrule{2-4} \cmidrule{6-8} 
					& \multicolumn{1}{c}{B} & \multicolumn{1}{c}{R} &\multicolumn{1}{c}{$\Delta \%$}&& \multicolumn{1}{c}{B} & \multicolumn{1}{c}{R}&\multicolumn{1}{c}{$\Delta \%$}&\\
					\midrule
					$\mu$ & 9552.722 & 2500.000 &-73.83\% && 0.423 & 0.116 &-72.63\% &\\
					$\bar{\mu}$ & 9551.000 & 2500.000& -73.83\% && 0.422 & 0.116& -72.55\%& \\
					$\sigma$ & 164.550 & 0 &-100.00\% && 0.009 & 0.001& -93.02\%& \\
					\hline
				\end{tabular}
			}
		}
	\end{center}
\end{table}
The computed KDE is truncated to $\mathcal{W}_{\mathrm{free}}^2$ and normalised, resulting in the KDE presented in Fig.~\ref{fig:toy_bounded_kde}.
By reducing the original data set $X$ to only encompass points which fall within $\mathcal{W}_{free}^2$, $\bar{X}$, the KDE can be sampled in either of the two ways presented in Section \ref{sec:samplekde}. 
Fig.~\ref{fig:toy_removing_samples} shows the subset $\bar{X}$, which will be utilised such that no rejection sampling is required. Finally, $\bar{X}$ and the resulting KDE are used to generate samples. 
Fig.~\ref{fig:toy_biased_sampling} and Fig.~\ref{fig:toy_uniform_sampling} demonstrate the sampling schemes ability to generate biased samples as well as an approximate uniform coverage of an extended region within $\mathcal{W}_{free}^1$ surrounding the original data set $X$.

\section{Case Studies}\label{sec:casestudies}
The following case studies are performed on raw data, without any sort of pre-processing or augmentation. This allows the effectiveness of the proposed method to be demonstrated. However, in practise, one may potentially gain further increased performance from procedures such as upsampling. 

All KDEs were computed using \texttt{KDEpy} \citep{odland2018kdepy}, an FFT-based KDE package for Python. The proposed rejection-free sampling method is compared to the simplest and most versatile method, namely uniform sampling. The planning problem is solved using RRT* \citep{karaman2011sampling}, but the sampling strategy is general and can therefore be used with other sampling-based motion planners. During each simulation in the comparison study, the various planner parameters remain constant; the only change is the sampling scheme. Both case studies are divided into two subproblems: (i) finding a feasible solution, and (ii) finding a solution at a lesser cost than some threshold. 
\begin{figure*}
	\centering
	\begin{subfigure}[b]{0.625\columnwidth}
		\centering
		\includegraphics[width=\textwidth]{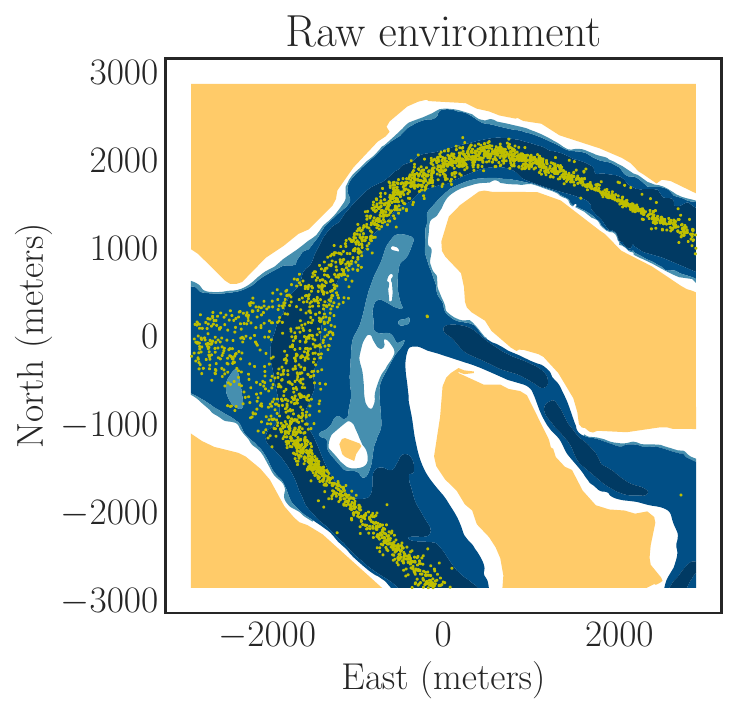}
		\caption{ $X$ and $\mathcal{W}^0_{\text{free}}$ }
		\label{fig:marine_raw}
	\end{subfigure}
	\begin{subfigure}[b]{0.625\columnwidth}
		\centering
		\includegraphics[width=\textwidth]{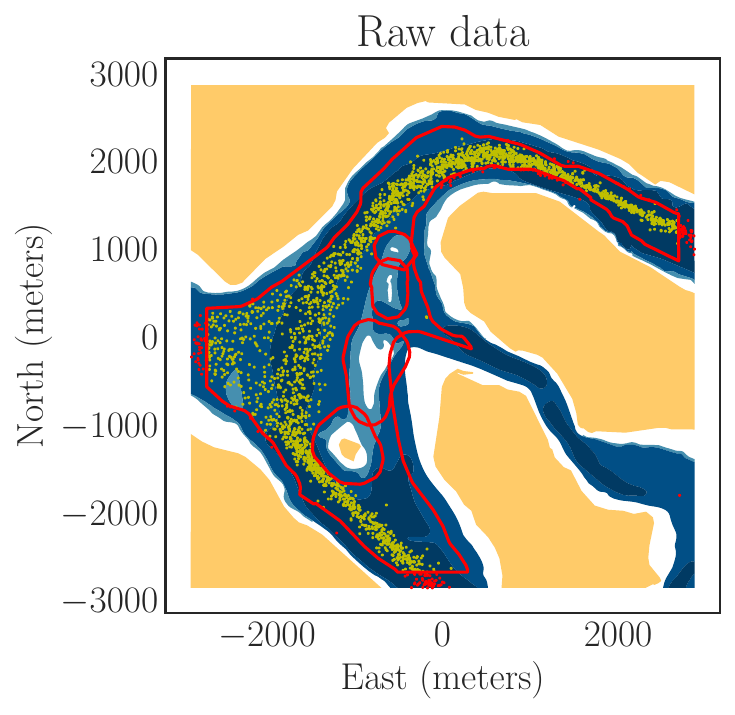}
		\caption{Imposing $\mathcal{W}^1_{\text{free}}$ on $X$}
		\label{fig:marine_raw_data}
	\end{subfigure}
	\begin{subfigure}[b]{0.625\columnwidth}
		\centering
		\includegraphics[width=\textwidth]{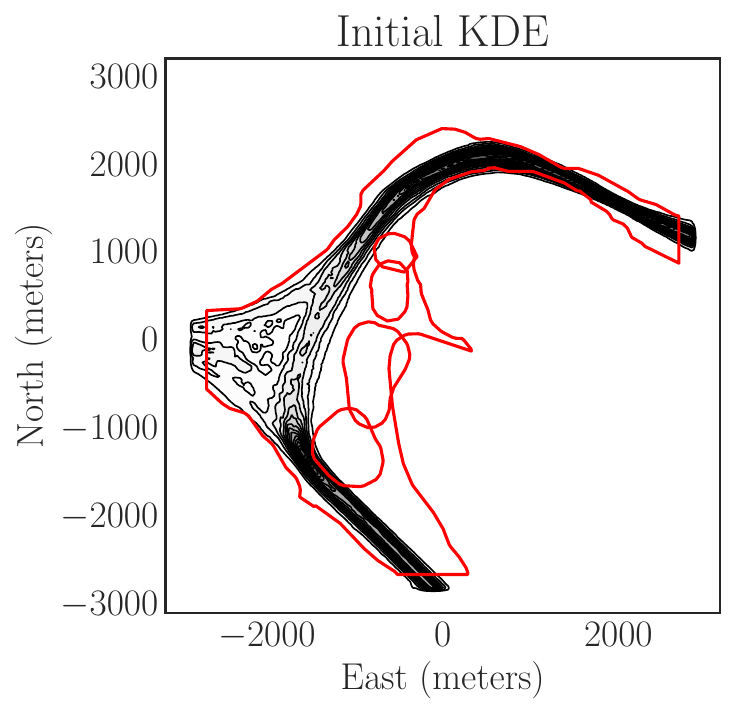}
		\caption{KDE based on data $X$}
		\label{fig:marine_initial_kde}
	\end{subfigure}
	\begin{subfigure}[b]{0.625\columnwidth}
		\centering
		\includegraphics[width=\textwidth]{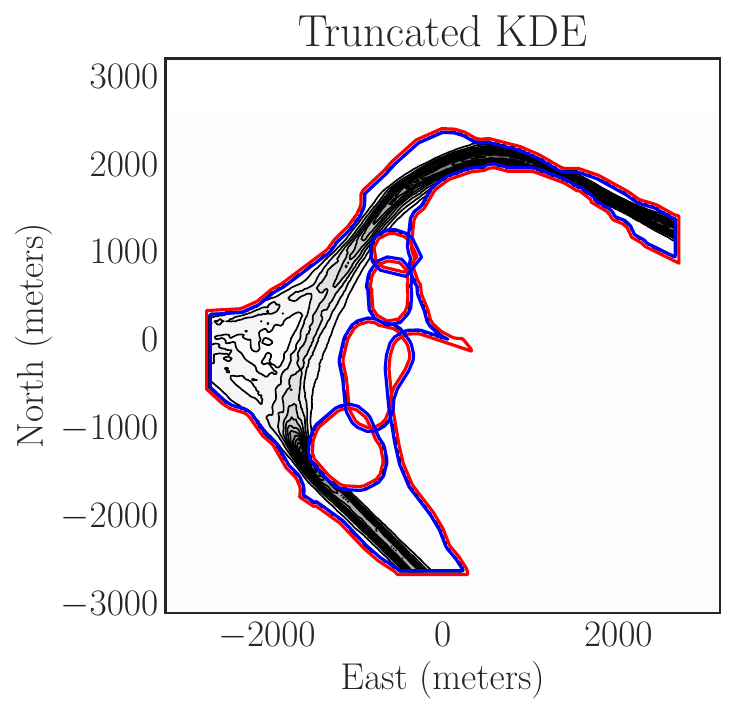}
		\caption{Erosion and dilation to create $\mathcal{W}^2_{\text{free}}$}
		\label{fig:marine_truncated_kde}
	\end{subfigure}
	\begin{subfigure}[b]{0.625\columnwidth}
		\centering
		\includegraphics[width=\textwidth]{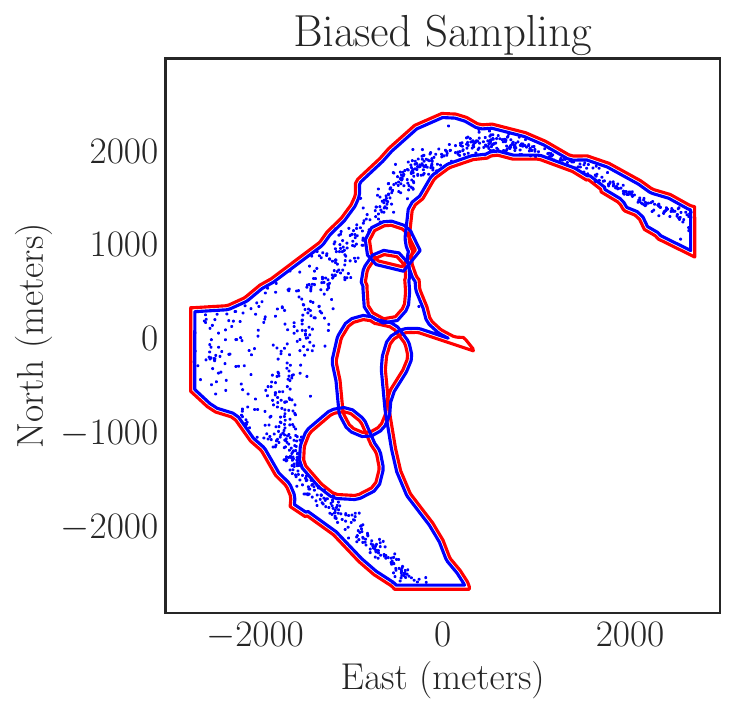}
		\caption{Sampling $\hat{f}_{\bar{X}}$ directly}
		\label{fig:marine_biased_sampling}
	\end{subfigure}
	\begin{subfigure}[b]{0.625\columnwidth}
		\centering
		\includegraphics[width=\textwidth]{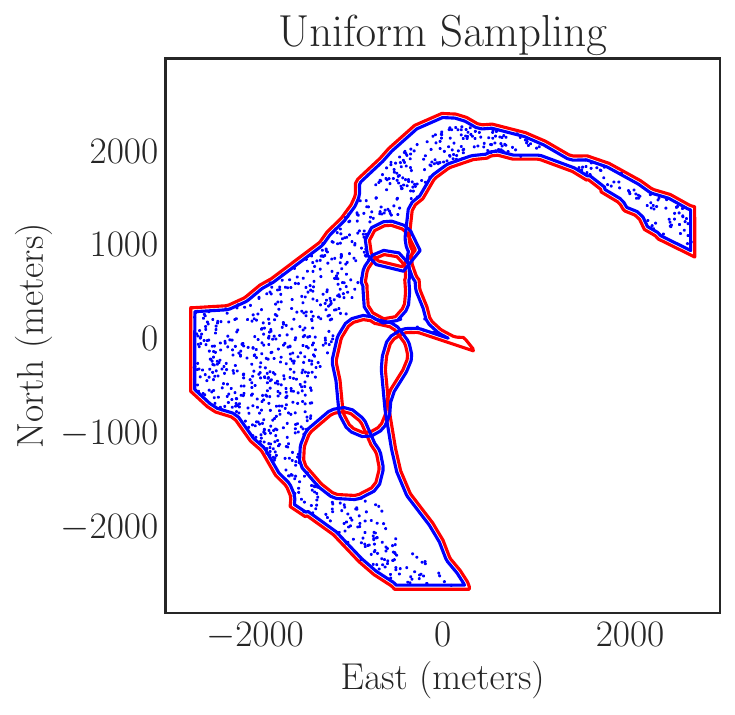}
		\caption{Sampling weighted $\hat{f}_{\bar{X}}$}
		\label{fig:marine_uniform_sampling}
	\end{subfigure}
	\caption{Application of the proposed method detailed in Section \ref{sec:propsed_method}, where the generated sampling schemes are created based on real historical data from vessels passing through the Little belt area of Denmark, such that the resulting samples ensure feasibility with respect to the available water depth. The white regions in Fig.~\ref{fig:marine_raw} and \ref{fig:marine_raw_data} are infeasible regions for the chosen vessel. For more details regarding this particular case study and associated data, see \cite{enevoldsen2021kde}.}
	\label{fig:marine_study}
%\end{figure*}
%\begin{figure*}
	\centering
	\begin{subfigure}[b]{0.625\columnwidth}
		\centering
		\includegraphics[width=\textwidth]{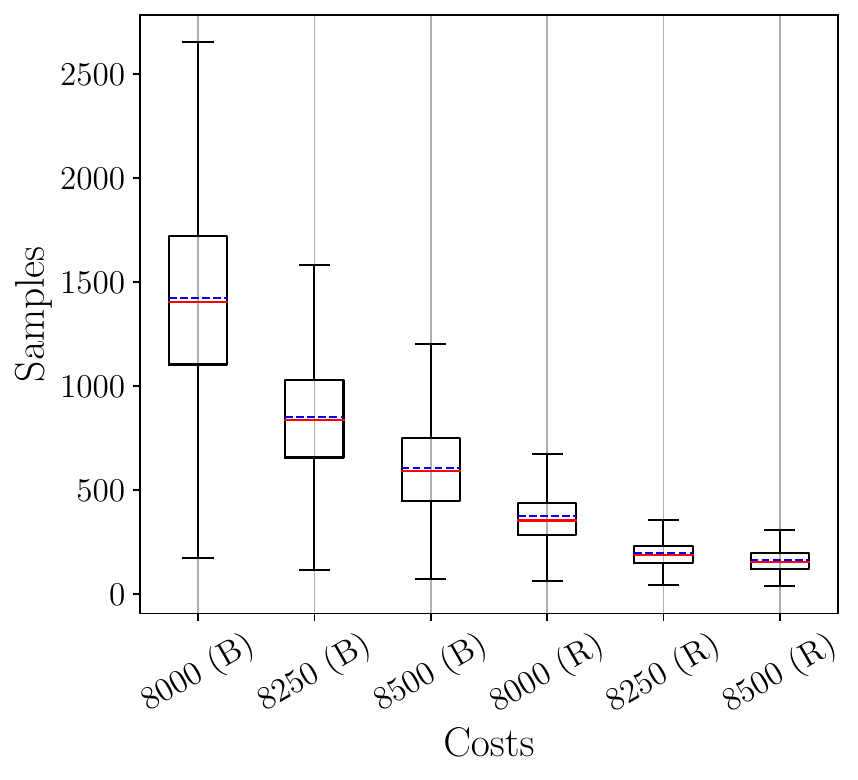}
		\caption{Number of samples}
		\label{fig:access_box_samples}
	\end{subfigure}
	\begin{subfigure}[b]{0.625\columnwidth}
		\centering
		\includegraphics[width=\textwidth]{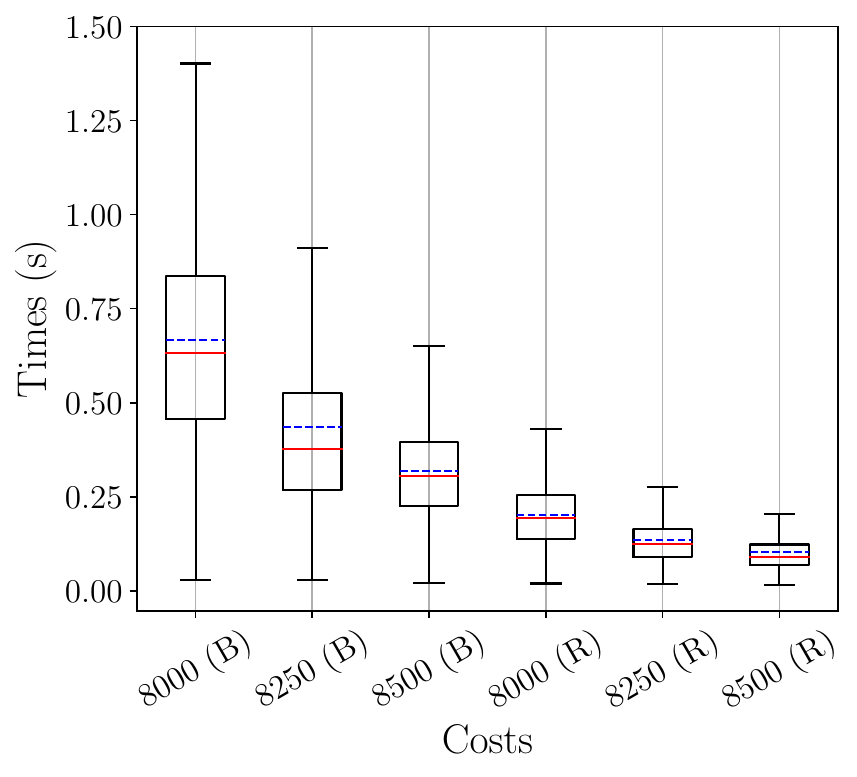}
		\caption{Computational time}
		\label{fig:access_box_times}
	\end{subfigure}
	\begin{subfigure}[b]{0.625\columnwidth}
		\centering
		\includegraphics[width=\textwidth]{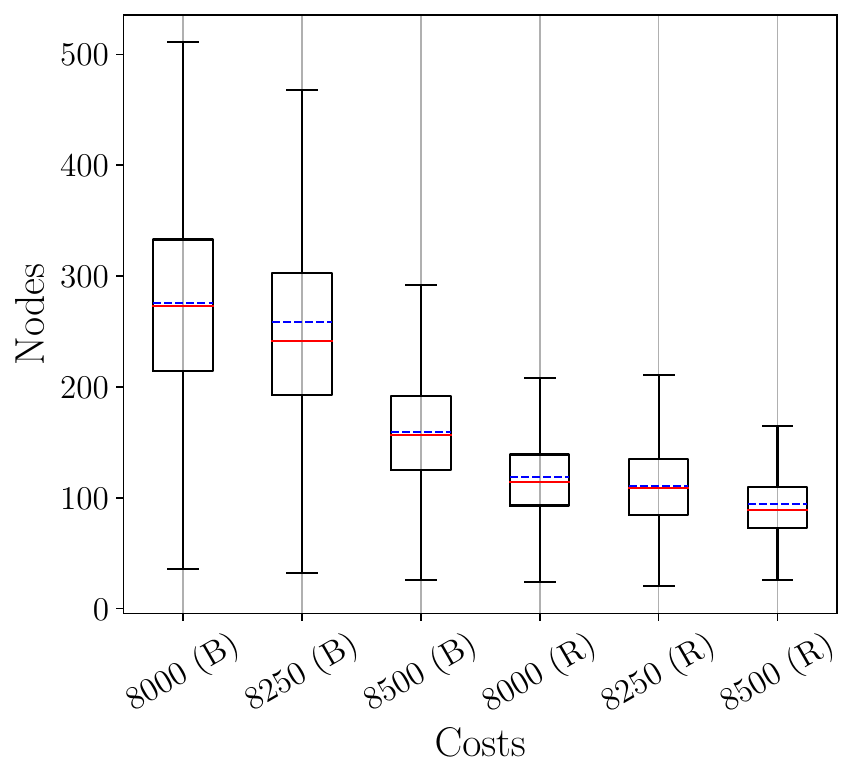}
		\caption{Graph nodes}
		\label{fig:access_box_nodes}
	\end{subfigure}
	\caption{Comparison between the required number of samples, computational time and graph nodes required to achieve the specified solution cost ($c=8000$, $c=8250$ and $c=8500$) for both the baseline sampling scheme (B) and proposed rejection-free sampling method (R), when solving the planning problem for the ASV (Fig.~\ref{fig:scenario_ASV}). The blue line indicates the mean and solid red the median. Each method has been simulated 10,000 times, further statistical results can be found in Table \ref{tab:stopping_costs_ASV}. }\label{fig:acces_box_plots}
\end{figure*}
\subsection{Autonomous Surface Vessel Sailing in Confined Waters}
The development of autonomous ships has been in focus in recent years, where it is desired to bring highly automated capabilities to vessels such as harbour buses, small island ferries, or even larger vessels such as container feeders. A key component towards the achievement of autonomous marine navigation is collision and grounding (i.e., sailing in waters shallower than the clearance) avoidance. Therefore, a tailored sampling space for such a system could be designed to directly sample regions where the given vessel would typically operate, without sampling states that may cause grounding. Formulating a data-driven sampling space for this particular application is made possible by a significant amount of available GPS data, since modern standards dictate that certain classes of vessels must broadcast their positions at all times. 
The presented case study (Fig.~\ref{fig:marine_study}) considers a vessel restricted by shallow waters, which, according to good and safe navigation practise, also wants to maintain a safe distance from shallow water \citep{enevoldsen2022oe}. The case study uses real ships' position and chart data from the Little Belt area in Denmark (for additional details regarding this particular case study and data see \cite{enevoldsen2021kde}). Fig.~\ref{fig:marine_raw} shows the feasible contours $\mathcal{W}^0_{\text{free}}$ and the past ships' position data $X$, and Fig.~\ref{fig:marine_raw_data} details the imposed safety distance, both on the data and contours ($\mathcal{W}^1_{\text{free}}$). Fig.~\ref{fig:marine_initial_kde} shows the estimated KDE based on the data set $X$ using the box kernel and $\mathbf{H} = 25\mathbf{I}$. Given the selected bandwidth, the boundary polygon is eroded and the obstacles are dilated, as detailed in Fig.~\ref{fig:marine_truncated_kde}, in order to generate the final space $\mathcal{W}^2_{\text{free}}$. Once obtained, the data set can be reduced by imposing $\mathcal{W}^2_{\text{free}}$ on $X$, giving $\bar{X}$, and then using it for guaranteed rejection-free sampling of $\mathcal{W}^1_{\text{free}}$. 
\begin{table*}
	\caption{Statistics related to computing feasible solutions to the planning problem related to the ASV 10,000 times (Fig.~\ref{fig:scenario_ASV}). Comparison between the baseline (B) and proposed approach (R), where the percentage ($\Delta \%$) is computed as $(R-B)/B$, where lower numbers are better. Mean ($\mu$), median ($\bar{\mu}$) and standard deviation ($\sigma$).}
	\label{tab:ASV_feasible}
	\begin{center}
		{\renewcommand{\arraystretch}{1.15}% for the vertical padding
			\resizebox{\textwidth}{!}{
				\begin{tabular}{rrrrcrrrcrrrcrrrr}
					\toprule
					& \multicolumn{3}{c}{Samples} & \phantom{a}&  \multicolumn{3}{c}{Times} & \phantom{a}&\multicolumn{3}{c}{Nodes} & \phantom{a}&\multicolumn{3}{c}{Cost}&\phantom{a}\\
					\cmidrule{2-4} \cmidrule{6-8} \cmidrule{10-12} \cmidrule{14-16}
					& \multicolumn{1}{c}{B} & \multicolumn{1}{c}{R}& \multicolumn{1}{c}{$\Delta \%$} && \multicolumn{1}{c}{B} & \multicolumn{1}{c}{R} & \multicolumn{1}{c}{$\Delta \%$}&& \multicolumn{1}{c}{B} & \multicolumn{1}{c}{R} & \multicolumn{1}{c}{$\Delta \%$}&& \multicolumn{1}{c}{B} & \multicolumn{1}{c}{R}& \multicolumn{1}{c}{$\Delta \%$}&\\
					\midrule
					$\mu$ & 318.832 & 156.470 &-50.92\% && 0.085 & 0.093 &10.07\% && 53.346 & 88.325 &65.57\% && 9011.607 & 8363.973 &-7.19\%& \\
					$\bar{\mu}$ & 288.000 & 147.000 &-48.96\% && 0.076 & 0.082 &8.90\% && 50.000 & 83.000& 66.00\% && 9001.567 & 8335.714 &-7.40\%& \\
					$\sigma$  & 155.558 & 60.195 &-61.30\% && 0.044 & 0.055& 23.86\% && 19.066 & 31.671& 66.11\% && 437.860 & 238.787 &-45.47\%& \\
					\hline
				\end{tabular}
			}
		}
	\end{center}
	\caption{Solutions to the ASV (10,000 times) planning problem (Fig.~\ref{fig:scenario_ASV}), comparing the baseline (B) and proposed approach (R), the planner was terminated once the solution cost $c$ was less than 8500, 8250, and 8000. The percentage ($\Delta \%$) is computed as $(R-B)/B$, where lower numbers are better. Mean ($\mu$), median ($\bar{\mu}$) and standard deviation ($\sigma$).} 
	\label{tab:stopping_costs_ASV}
	\begin{center}
		{\renewcommand{\arraystretch}{1.15}% for the vertical padding
			\resizebox{\textwidth}{!}{
				\begin{tabular}{rrrrcrrrcrrrcrrrr}
					\toprule
					& \multicolumn{3}{c}{Samples} & \phantom{a}&  \multicolumn{3}{c}{Times} & \phantom{a}&\multicolumn{3}{c}{Nodes} & \phantom{a}&\multicolumn{3}{c}{Cost}&\phantom{a}\\
					\cmidrule{2-4} \cmidrule{6-8} \cmidrule{10-12} \cmidrule{14-16}
					& \multicolumn{1}{c}{B} & \multicolumn{1}{c}{R}& \multicolumn{1}{c}{$\Delta \%$} && \multicolumn{1}{c}{B} & \multicolumn{1}{c}{R} & \multicolumn{1}{c}{$\Delta \%$}&& \multicolumn{1}{c}{B} & \multicolumn{1}{c}{R} & \multicolumn{1}{c}{$\Delta \%$}&& \multicolumn{1}{c}{B} & \multicolumn{1}{c}{R}& \multicolumn{1}{c}{$\Delta \%$}&\\
					\midrule
					\multicolumn{16}{c}{$c = 8500$}\\
					$\mu$ & 606.060 & 163.559 & -73.01\% & & 0.202 & 0.103 & -49.01\% & & 110.715 & 94.495 & -14.65\% & & 8406.989 & 8301.497 & -1.25\% & \\
					$\bar\mu$ & 592.000 & 154.000 & -73.99\% & & 0.192 & 0.091 & -52.75\% & & 109.000 & 89.000 & -18.35\% & & 8443.666 & 8324.918 & -1.41\% & \\
					$\sigma$ & 234.630 & 63.259 & -73.04\% & & 0.091 & 0.075 & -17.76\% & & 38.008 & 35.274 & -7.19\% & & 122.027 & 153.021 & 25.40\% & \\
					\hline
					\multicolumn{16}{c}{$c = 8250$}\\
					$\mu$ & 849.348 & 194.905 & -77.05\% & & 0.318 & 0.136 & -57.24\% & & 159.592 & 118.532 & -25.73\% & & 8185.689 & 8179.320 & -0.08\% & \\
					$\bar\mu$ & 834.000 & 188.000 & -77.46\% & & 0.305 & 0.124 & -59.25\% & & 157.000 & 114.000 & -27.39\% & & 8213.833 & 8204.987 & -0.11\% & \\
					$\sigma$ & 286.791 & 72.240 & -74.81\% & & 0.134 & 0.116 & -13.30\% & & 51.661 & 43.497 & -15.80\% & & 92.511 & 83.872 & -9.34\% & \\
					\hline
					\multicolumn{16}{c}{$c = 8000$}\\
					$\mu$ & 1423.857 & 373.699 & -73.75\% & & 0.666 & 0.434 & -34.82\% & & 275.855 & 258.625 & -6.25\% & & 7948.200 & 7979.359 & 0.39\% & \\
					$\bar\mu$ & 1405.000 & 353.000 & -74.88\% & & 0.633 & 0.378 & -40.34\% & & 273.000 & 241.500 & -11.54\% & & 7978.459 & 7987.719 & 0.12\% & \\
					$\sigma$ & 478.265 & 139.103 & -70.92\% & & 0.301 & 0.264 & -12.27\% & & 93.382 & 102.636 & 9.91\% & & 81.168 & 39.102 & -51.83\% & \\
					\hline
				\end{tabular}
			}
		}
	\end{center}
\end{table*}
Given the reduced data set $\bar{X}$ and the truncated KDE, samples can be generated directly from the estimated distribution. However, this results in samples that are biased towards the original dataset, which for motion planning applications may be undesirable. Therefore, samples are weighted by the inverse of their probability density to generate approximately uniform samples of the domain described by the truncated KDE. The biased sampling and approximately uniform sampling is shown in Fig.~\ref{fig:marine_biased_sampling} and \ref{fig:marine_uniform_sampling} respectively, where the only difference is how samples are generated from the KDE.

Monte Carlo simulations were used to investigate the performance of the proposed method for generating approximately uniform samples from the KDE. All the following results were generated from solving the planning problem detailed in Fig.~\ref{fig:scenario_ASV}, where the baseline sampling scheme is simply a rectangular approximation of the planning region. 
\begin{figure*}
	\centering
	\begin{subfigure}[b]{0.625\columnwidth}
		\centering
		\includegraphics[width=\textwidth]{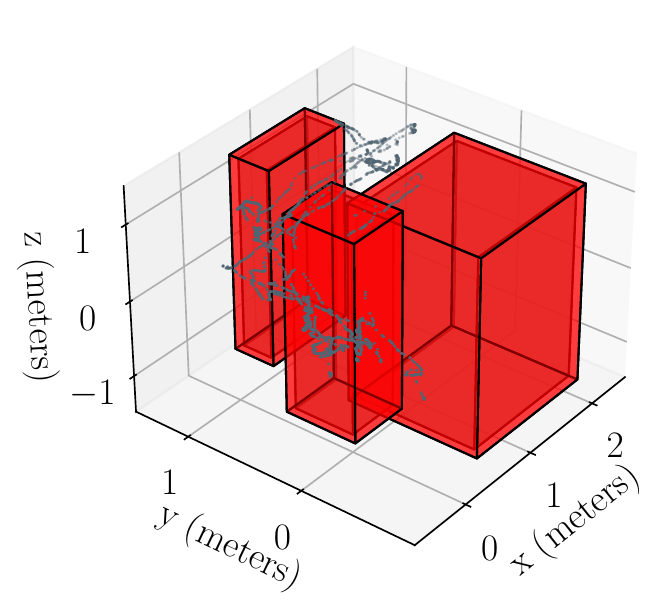}
		\caption{Imposing $\mathcal{W}^1_{\text{free}}$ on $X$}
		\label{fig:drone_3d}
	\end{subfigure}
	\begin{subfigure}[b]{0.625\columnwidth}
		\centering
		\includegraphics[width=\textwidth]{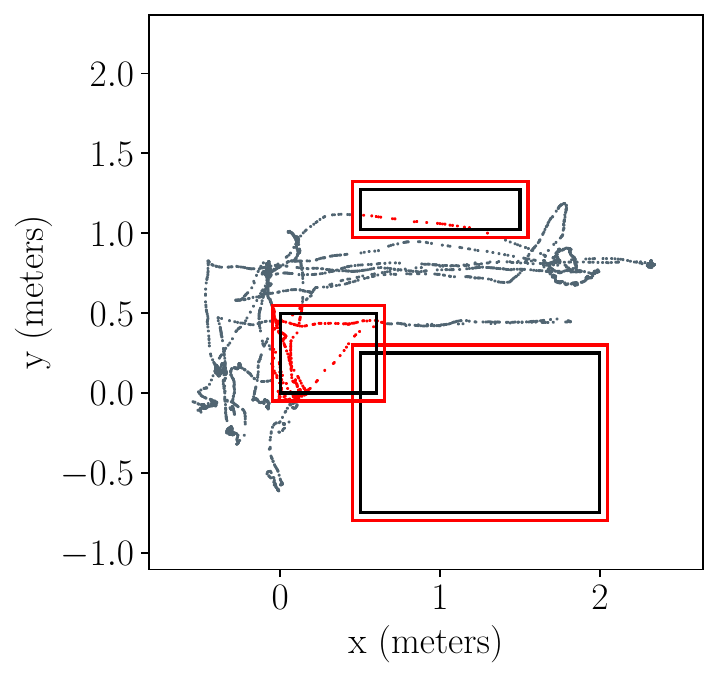}
		\caption{Imposing $\mathcal{W}^1_{\text{free}}$ on $X$}
		\label{fig:drone_2d}
	\end{subfigure}
	\begin{subfigure}[b]{0.625\columnwidth}
		\centering
		\includegraphics[width=\textwidth]{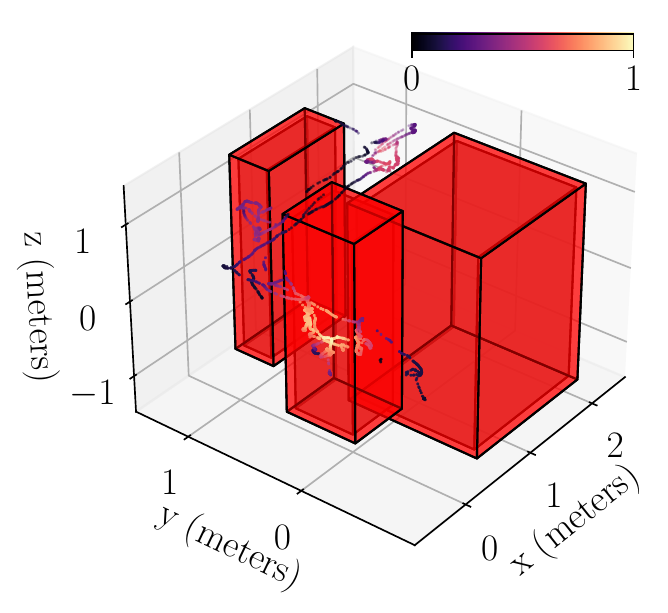}
		\caption{KDE based on data $X$}
		\label{fig:drone_kde_3d}
	\end{subfigure}
	\begin{subfigure}[b]{0.625\columnwidth}
		\centering
		\includegraphics[width=\textwidth]{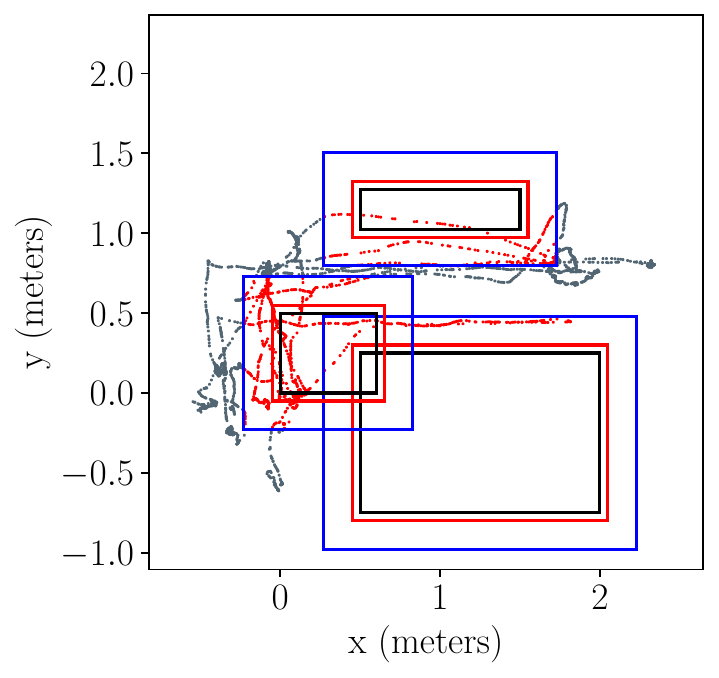}
		\caption{Erosion and dilation to create $\mathcal{W}^2_{\text{free}}$}
		\label{fig:drone_inflated_bw_2d}
	\end{subfigure}
	\begin{subfigure}[b]{0.625\columnwidth}
		\centering
		\includegraphics[width=\textwidth]{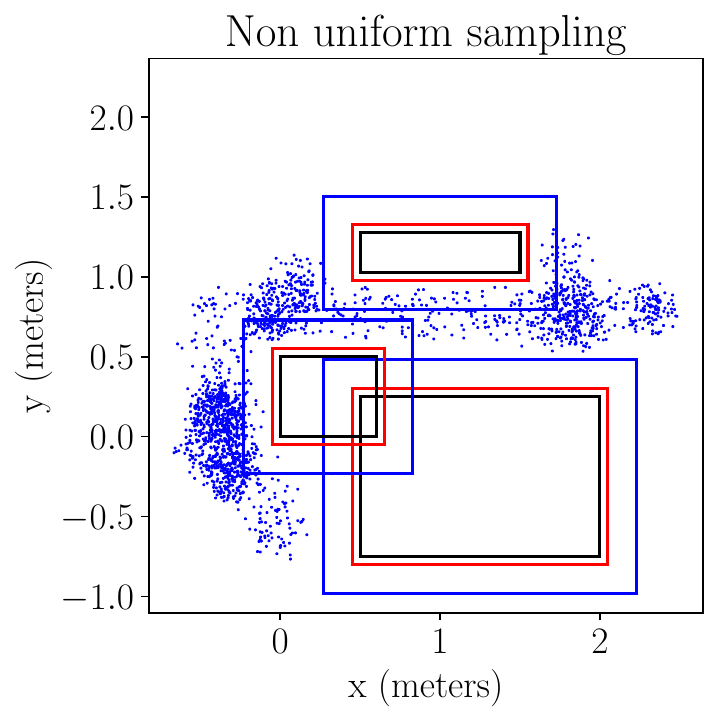}
		\caption{Sampling $\hat{f}_{\bar{X}}$ directly}
		\label{fig:drone_non_uniform_2d}
	\end{subfigure}
	\begin{subfigure}[b]{0.625\columnwidth}
		\centering
		\includegraphics[width=\textwidth]{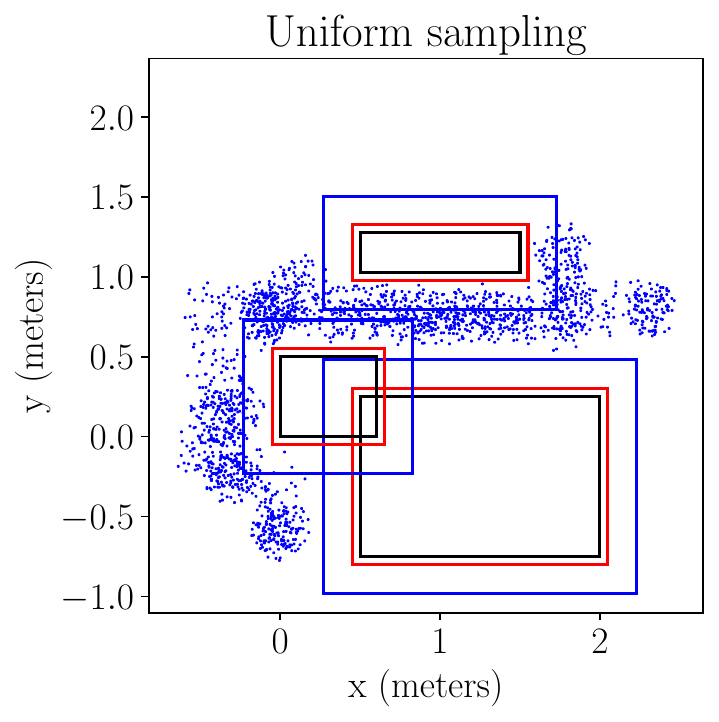}
		\caption{Sampling weighted $\hat{f}_{\bar{X}}$}
		\label{fig:drone_uniform_2d}
	\end{subfigure}
	\caption{Application of the proposed sampling method from Section \ref{sec:propsed_method}, with real life data from a drone inspecting a confined space~\citep{inspectdrone_data}. The resulting space ensures that the drone is able to maintain an adequate safety distance from the updated obstacle space. Fig.~\ref{fig:drone_2d} and \ref{fig:drone_inflated_bw_2d}-\ref{fig:drone_uniform_2d} are the 2D projections to the $xy$-plane of the space shown in Fig.~\ref{fig:drone_3d}, for $z=0$. Fig.~\ref{fig:drone_kde_3d} is a 3D visualisation of the untruncated KDE; the colour of each data point indicates the probability density. Note that due to limitations with the 3D engine used for plotting, the location of the data points in Fig.~\ref{fig:drone_3d} and Fig.~\ref{fig:drone_kde_3d} may be deceptive.}
	\label{fig:drone_study}
%\end{figure*}
%\begin{figure*}
	\centering
	\begin{subfigure}[b]{0.625\columnwidth}
		\centering
		\includegraphics[width=\textwidth]{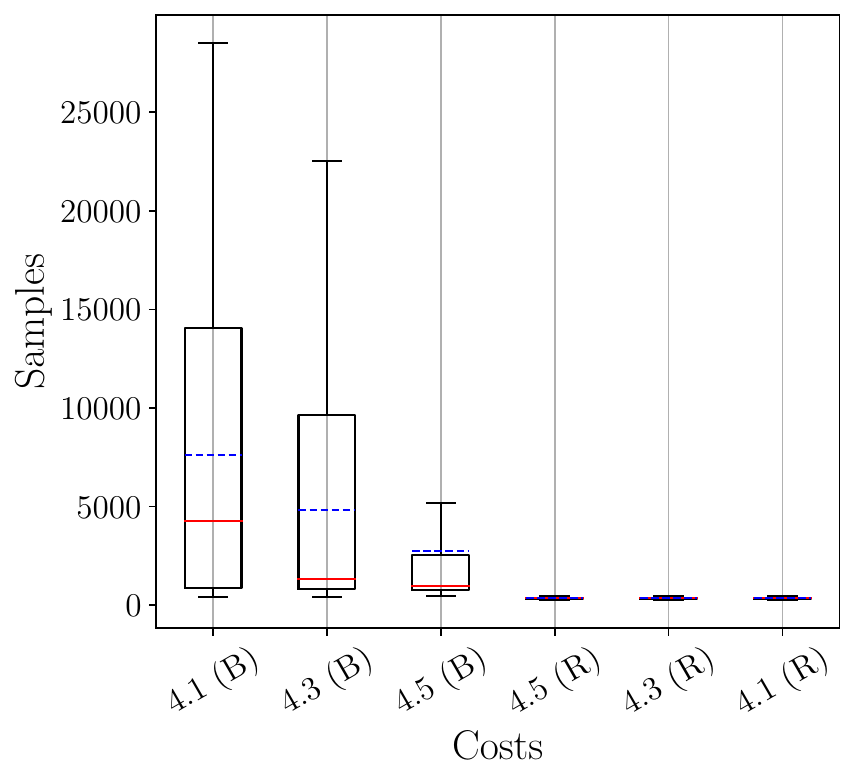}
		\caption{Number of samples}
		\label{fig:access_box_drone_samples}
	\end{subfigure}
	\begin{subfigure}[b]{0.625\columnwidth}
		\centering
		\includegraphics[width=\textwidth]{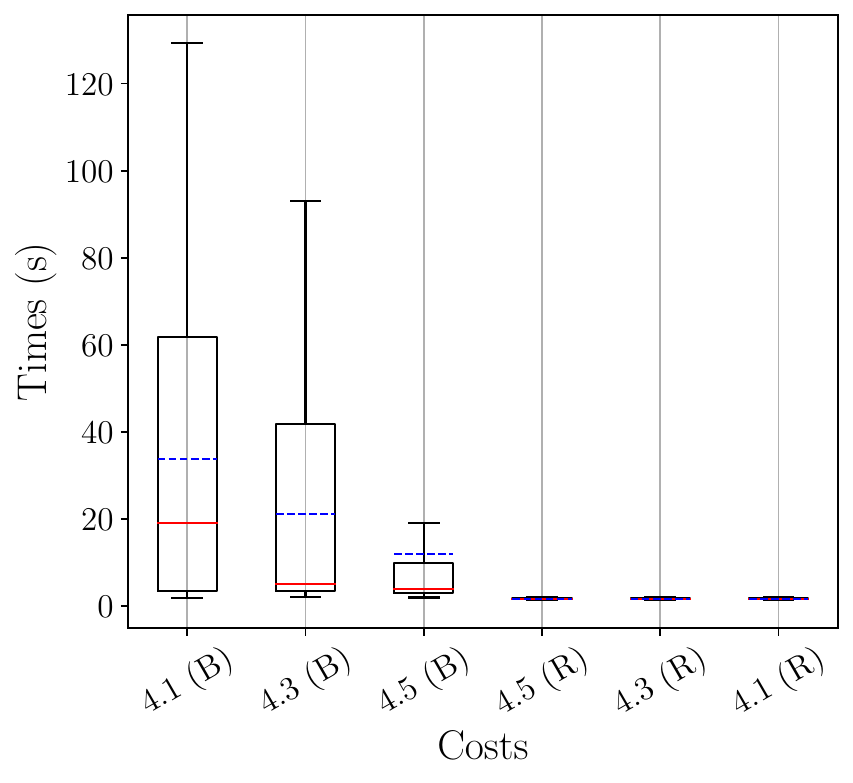}
		\caption{Computational time}
		\label{fig:access_box_drone_times}
	\end{subfigure}
	\begin{subfigure}[b]{0.625\columnwidth}
		\centering
		\includegraphics[width=\textwidth]{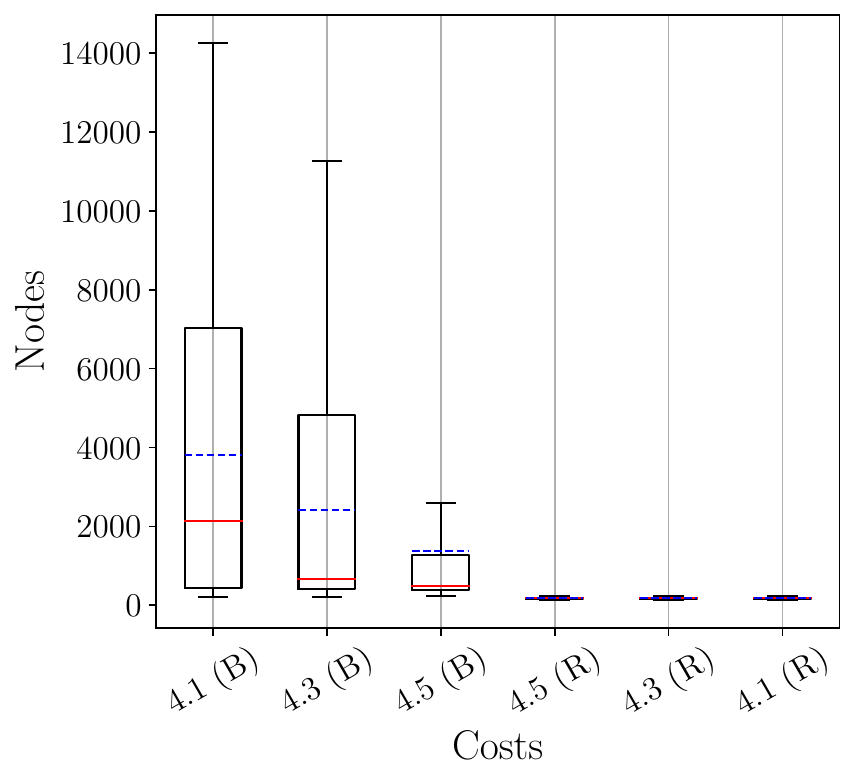}
		\caption{Graph nodes}
		\label{fig:access_box_drone_nodes}
	\end{subfigure}
	\caption{Comparison between the required number of samples, computational time and graph nodes required to achieve the specified solution cost ($c=4.1$, $c=4.3$ and $c=4.5$) for both the baseline sampling scheme (B) and proposed rejection-free sampling method (R), when solving the drone planning problem in Fig.~\ref{fig:drone_scenario_path} and Fig.~\ref{fig:drone_scenario_path_2d}. The blue line indicates the mean and solid red the median. Each method has been simulated 500 times, further statistical results can be found in Table \ref{tab:stopping_costs_drone}.}\label{fig:acces_box_drone_plots_drone}
\end{figure*}

Table \ref{tab:ASV_sampling_speeds} details 10,000 trials, where the two sampling schemes were tasked with generating 2500 feasible samples. As expected, since the proposed method directly samples the free space in a rejection-free manner, each trial spends exactly 2500 samples in order to generate the required 2500 feasible samples. The performance of the baseline sampler is highly correlated with the area ratio between the actual free space and its approximation. In general, the traditional uniform sampler results in a constant factor, relative to the area ratio, in terms of increased run time~\citep{bialkowski2013free,enevoldsen2021kde}. These results highlight the sampling efficiency to be gained from specialising the scheme to the particular problem at hand, as the strength of the method becomes more apparent when the baseline sampler is a poor approximation of the free space.  Table \ref{tab:ASV_feasible} contains the results of 10,000 trials in which the baseline and the proposed method are tasked with simply computing a feasible solution. On average, the proposed approach is 10\% slower at finding an initial solution; however, the achieved solution cost is 7\% lower. The results show that the graph density (number of nodes) for the proposed method is 65\% higher for this particular case study, which explains why the overall run time is increased. 

The true strength of the proposed method is captured when one wishes to find paths within the space of past behaviours, which are lower cost than some desired cost. Table \ref{tab:stopping_costs_ASV} shows 10,000 trials for both methods, where both planners solve the problem in Fig.~\ref{fig:scenario_ASV} until the reported solution cost is lower than 8500, 8250 and 8000. Fig.~\ref{fig:acces_box_plots} visualises the table data, where it is clear that to obtain a specific solution cost, the rejection-free method excels. The proposed method for all three cost scenarios outperform every baseline, whilst also significantly reducing the amount of standard deviation in the resulting solution.
As previously detailed, using the proposed method allows one to successfully capture the underlying navigational behaviour of past vessels. The sampling scheme guarantees that any generated sample is approximately uniform in the domain of the KDE, which describes the past behaviours of other vessels, and always falls within the defined free space. The approach allows one to include design specifications directly in the formulation of the sampling space, which in this case translate to maintaining a certain safety distance towards the shallow waters.
\begin{table*}
	\caption{Statistics related to computing feasible solutions to the drone planning problem 10,000 times (Fig.~\ref{fig:drone_scenario_path}). Comparison between the baseline (B) and proposed approach (R), where the percentage ($\Delta \%$) is computed as $(R-B)/B$, where lower numbers are better. Mean ($\mu$), median ($\bar{\mu}$) and standard deviation ($\sigma$).}
	\label{tab:drone_feasible}
	\begin{center}
		{\renewcommand{\arraystretch}{1.15}% for the vertical padding
			\resizebox{\textwidth}{!}{
				\begin{tabular}{rrrrcrrrcrrrcrrrr}
					\toprule
					& \multicolumn{3}{c}{Samples} & \phantom{a}&  \multicolumn{3}{c}{Times} & \phantom{a}&\multicolumn{3}{c}{Nodes} & \phantom{a}&\multicolumn{3}{c}{Cost}&\phantom{a}\\
					\cmidrule{2-4} \cmidrule{6-8} \cmidrule{10-12} \cmidrule{14-16}
					& \multicolumn{1}{c}{B} & \multicolumn{1}{c}{R}& \multicolumn{1}{c}{$\Delta \%$} && \multicolumn{1}{c}{B} & \multicolumn{1}{c}{R} & \multicolumn{1}{c}{$\Delta \%$}&& \multicolumn{1}{c}{B} & \multicolumn{1}{c}{R} & \multicolumn{1}{c}{$\Delta \%$}&& \multicolumn{1}{c}{B} & \multicolumn{1}{c}{R}& \multicolumn{1}{c}{$\Delta \%$}&\\
					\midrule
					$\mu$ & 1001.151 & 355.098 &-64.35\%&& 4.018 & 1.701 &-57.66\%&& 502.026 & 178.974 &-64.35\%&& 4.268 & 3.670 &-14.02\%& \\
					$\bar{\mu}$ & 871.000 & 351.000 &-59.45\%&& 3.555 & 1.691 &-52.44\%&& 436.500 & 177.000 &-59.45\%&& 4.194 & 3.664 &-12.63\%& \\
					$\sigma$ & 457.672 & 44.332 &-90.31\%&& 1.636 & 0.163 &-90.01\%&& 228.824 & 22.174 &-90.31\%&& 0.402 & 0.048 &-88.15\%& \\
					\hline
				\end{tabular}
			}
		}
	\end{center}

	\caption{Drone planning problem results (500 simulations), Fig.~\ref{fig:drone_scenario_path} and Fig.~\ref{fig:drone_scenario_path_2d}, the planner was terminated once the cost $c$ was less than 4.5, 4.3, and 4.1. Comparison between the baseline (B) and the proposed approach (R), where the percentage ($\Delta \%$) is computed as $(R-B)/B$, where lower numbers are better. Mean ($\mu$), median ($\bar{\mu}$) and standard deviation ($\sigma$).
	} 
	\label{tab:stopping_costs_drone}
	\begin{center}
		{\renewcommand{\arraystretch}{1.15}% for the vertical padding
			\resizebox{\textwidth}{!}{
				\begin{tabular}{rrrrcrrrcrrrcrrrr}
					\toprule
					& \multicolumn{3}{c}{Samples} & \phantom{a}&  \multicolumn{3}{c}{Times} & \phantom{a}&\multicolumn{3}{c}{Nodes} & \phantom{a}&\multicolumn{3}{c}{Cost}&\phantom{a}\\
					\cmidrule{2-4} \cmidrule{6-8} \cmidrule{10-12} \cmidrule{14-16}
					& \multicolumn{1}{c}{B} & \multicolumn{1}{c}{R}& \multicolumn{1}{c}{$\Delta \%$} && \multicolumn{1}{c}{B} & \multicolumn{1}{c}{R} & \multicolumn{1}{c}{$\Delta \%$}&& \multicolumn{1}{c}{B} & \multicolumn{1}{c}{R} & \multicolumn{1}{c}{$\Delta \%$}&& \multicolumn{1}{c}{B} & \multicolumn{1}{c}{R}& \multicolumn{1}{c}{$\Delta \%$}&\\
					\midrule
					\multicolumn{17}{c}{$c = 4.5$}\\
					$\mu$ & 2750.836 & 355.300 & -87.08\% & & 11.933 & 1.630 & -86.34\% & & 1376.764 & 179.096 & -86.99\% & & 4.162 & 3.669 & -11.86\% & \\
					$\bar\mu$ & 993.000 & 349.500 & -64.80\% & & 3.888 & 1.614 & -58.49\% & & 498.000 & 176.000 & -64.66\% & & 4.204 & 3.661 & -12.90\% & \\
					$\sigma$ & 3430.622 & 44.749 & -98.70\% & & 15.748 & 0.146 & -99.07\% & & 1715.143 & 22.374 & -98.70\% & & 0.255 & 0.048 & -81.11\% & \\
					\hline
					\multicolumn{17}{c}{$c = 4.3$}\\
					$\mu$ & 4815.408 & 354.742 & -92.63\% & & 21.178 & 1.643 & -92.24\% & & 2408.996 & 178.782 & -92.58\% & & 4.077 & 3.669 & -10.02\% & \\
					$\bar\mu$ & 1312.000 & 352.000 & -73.17\% & & 4.994 & 1.633 & -67.30\% & & 657.500 & 177.500 & -73.00\% & & 4.128 & 3.662 & -11.29\% & \\
					$\sigma$ & 5284.368 & 45.618 & -99.14\% & & 24.051 & 0.159 & -99.34\% & & 2641.956 & 22.800 & -99.14\% & & 0.194 & 0.048 & -75.32\% & \\
					\hline
					\multicolumn{17}{c}{$c = 4.1$}\\
					$\mu$ & 7623.002 & 354.372 & -95.35\% & & 33.850 & 1.647 & -95.13\% & & 3812.696 & 178.608 & -95.32\% & & 3.988 & 3.673 & -7.91\% & \\
					$\bar\mu$ & 4271.000 & 350.000 & -91.81\% & & 18.986 & 1.640 & -91.36\% & & 2137.000 & 176.500 & -91.74\% & & 4.026 & 3.667 & -8.93\% & \\
					$\sigma$ & 7680.157 & 45.064 & -99.41\% & & 34.972 & 0.155 & -99.56\% & & 3839.816 & 22.544 & -99.41\% & & 0.112 & 0.048 & -57.58\% & \\
					\hline
				\end{tabular}
			}
		}
	\end{center}
\end{table*}

\subsection{Navigating Autonomous Drones in Challenging Environments}
Drones are widely adopted in various industries, proving their worth in many highly automated or autonomous tasks. 
Their ability to move freely in 3D provides significant value when it comes to performing inspection, monitoring, and surveying tasks, as well as providing the ability to interact and reach places infeasible for humans. This particular case study was chosen to demonstrate the applicability of the proposed methods to higher-dimensional problems. This case study uses data from a real drone that has previously inspected a marine vessel, specifically ballast tank inspection \citep{inspectdrone_data}. The motion planning problem is then to compute a path for a new inspection task.

Fig.~\ref{fig:drone_study} details the steps of the proposed method. Given some inspection data collected in $\mathcal{W}^0_{\text{free}}$, the problem changes to contain new obstacles and a safety distance requirement, which generates $\mathcal{W}^1_{\text{free}}$ (Fig.~\ref{fig:drone_3d} and Fig.~\ref{fig:drone_2d}). Next, the bandwidth matrix $\mathbf{H} = 0.18\mathbf{I}$ and Epanechnikov kernel is selected and the KDE is computed and evaluated (Fig.~\ref{fig:drone_kde_3d}), where then the space $\mathcal{W}^1_{\text{free}}$ is dilated by the bandwidth, generating $\mathcal{W}^2_{\text{free}}$ (Fig.~\ref{fig:drone_inflated_bw_2d}). By imposing $\mathcal{W}^2_{\text{free}}$ on $X$, one can generate biased or approximately uniform samples as desired (Fig.~\ref{fig:drone_non_uniform_2d} and Fig.~\ref{fig:drone_uniform_2d}).

Compared to the ASV case study, when looking for a feasible solution to the drone scenario, the proposed method significantly outperforms the baseline. Table \ref{tab:drone_feasible} details 10,000 simulations of the presented planning problem in Fig.~\ref{fig:drone_scenario_path} and Fig.~\ref{fig:drone_scenario_path_2d}. Here, it is evident that as the dimensional complexity increases, the true benefits of a specialised sampling scheme appear. Since the particular historical data are a good representation of the desired path, the proposed method is faster (by 57\%) and produces lower cost paths at a significantly lower standard deviation (88\% lower).
Using the proposed method, the rate for obtaining lower-cost solutions for the drone navigation problem also heavily outperforms the baseline. Three cost ($c=4.5$, $c=4.3$ and $c=4.1$) thresholds were selected that lie within the range of what was produced by the feasible solutions of both methods (from Table \ref{tab:drone_feasible}). Due to the sheer amount of compute time required by the baseline, the number of trials was reduced to 500. Table \ref{tab:stopping_costs_drone} and Fig.~\ref{fig:acces_box_drone_plots_drone} detail the statistics related to obtaining a solution better than the cost thresholds mentioned above. Since the gathered historical data is very representative of the desired paths, the proposed method performs significantly better than the baseline. The baseline spends significant sampling effort exploring the entire space, where, instead, the rejection-free scheme hones its search. This results in greater speeds and overall consistency of the proposed method, since, as shown by the median and mean values of the baseline, there are a large number of outliers in the 500 trials that significantly impact the performance. By comparing the median values, the proposed method performs at least twice as well for $c=4.5$ and $c=4.3$, while for $c=4.1$ it performs 11 times better.

\section{Conclusions}
This paper proposed a new learning-based sampling strategy to generate biased or approximately uniform samples of a given free space, while guaranteeing that no rejection sampling is required. Kernel density estimation is adopted to achieve a probabilistic non-parametric description of regions of the workspace where solutions of the motion planning problem are likely to exist. The kernel and bandwidth of the estimated kernel density were exploited to provide a guarantee-by-construction that all future states generated by sampling the KDE fall within the boundaries of the free space. The method was illustrated in two case studies using real historical data for 2D and 3D workspaces, collected from surface vessels and drones, respectively. Each case study contained detailed steps explaining how to generate the sampling spaces. Finally, motion planning problems were solved for both the ASV and the drone case studies, and extensive Monte Carlo simulations were performed to gather statistical data, which detailed the strengths of the method.

\textcolor{\revOne}{
In comparison to the existing literature on learning-based sampling algorithms for SBMPs, which leans towards generating biased samples, the proposed method also allows the guaranteed rejection-free generation of approximately uniform samples of the free space. The importance of being able to generate uniform samples presents itself especially when basing a sampling scheme on historical data, as certain aspects of the data from past experiences may be greatly over-represented, since certain systems may have very repetitive operative or behavioural patterns over time. Using the proposed method for generating approximately uniform samples of the free space ensures that the spatial distribution of the past data is uniformly covered, when generating new samples, and thereby enables the SBMP to explore the free space without being biased towards the densest part of the data.}

It should be noted that the proposed method performs only as well as the available historical data. If the sought motion plan does not exist within the neighbourhood of the past behaviours, then the proposed method will perform worse than the baseline uniform sampling scheme.

Potential future work includes investigating whether or not pre-processing the data provides any performance improvement. It is speculated that for cases with highly application specific data, data augmentation or modification may improve the performance. Additionally, the statistical properties of the approximately uniform sampling scheme could be further investigated such that the uniformity of the sampled points could be accessed with respect to the underlying distribution.

%-------------------------------

\section*{CRediT Authorship Contribution Statement}
\textbf{Thomas T. Enevoldsen}: Conceptualization, Methodology, Software, Simulation, Validation, Investigation, Writing - Original Draft.
\textbf{Roberto Galeazzi}: Conceptualization, Writing - Review \& Editing, Supervision, Funding acquisition.

\section*{Acknowledgements}
The authors acknowledge the Danish Innovation Fund, The Danish Maritime Fund, Orients Fund and the Lauritzen Foundation for support to the ShippingLab Project, Grand Solutions 8090-00063B, which has sponsored this research. The electronic navigational charts have been provided by the Danish Geodata Agency. The drone data have been supplied by the Inspectrone project. Furthermore, the authors express their gratitude to Prof. Line K. H. Clemmensen for inspiring discussions on non-parametric estimation methods.

%\appendix

\bibliographystyle{cas-model2-names} 
\bibliography{cas-refs}

\begin{thebibliography}{56}
\expandafter\ifx\csname natexlab\endcsname\relax\def\natexlab#1{#1}\fi
\providecommand{\url}[1]{\texttt{#1}}
\providecommand{\href}[2]{#2}
\providecommand{\path}[1]{#1}
\providecommand{\DOIprefix}{doi:}
\providecommand{\ArXivprefix}{arXiv:}
\providecommand{\URLprefix}{URL: }
\providecommand{\Pubmedprefix}{pmid:}
\providecommand{\doi}[1]{\href{http://dx.doi.org/#1}{\path{#1}}}
\providecommand{\Pubmed}[1]{\href{pmid:#1}{\path{#1}}}
\providecommand{\bibinfo}[2]{#2}
\ifx\xfnm\relax \def\xfnm[#1]{\unskip,\space#1}\fi
%Type = Inproceedings
\bibitem[{Abbasi-Yadkori et~al.(2017)Abbasi-Yadkori, Bartlett, Gabillon and Malek}]{abbasi2017hit}
\bibinfo{author}{Abbasi-Yadkori, Y.}, \bibinfo{author}{Bartlett, P.}, \bibinfo{author}{Gabillon, V.}, \bibinfo{author}{Malek, A.}, \bibinfo{year}{2017}.
\newblock \bibinfo{title}{Hit-and-run for sampling and planning in non-convex spaces}, in: \bibinfo{booktitle}{Artificial Intelligence and Statistics}, \bibinfo{organization}{PMLR}. pp. \bibinfo{pages}{888--895}.
%Type = Inproceedings
\bibitem[{Andersen et~al.(2022)Andersen, Zajaczkowski, Jaiswal, Xu, Fan and Boukas}]{inspectdrone_data}
\bibinfo{author}{Andersen, R.E.}, \bibinfo{author}{Zajaczkowski, M.}, \bibinfo{author}{Jaiswal, H.}, \bibinfo{author}{Xu, J.}, \bibinfo{author}{Fan, W.}, \bibinfo{author}{Boukas, E.}, \bibinfo{year}{2022}.
\newblock \bibinfo{title}{Depth-based deep learning for manhole detection in uav navigation}, in: \bibinfo{booktitle}{2022 IEEE International Conference on Imaging Systems and Techniques (IST)}, pp. \bibinfo{pages}{1--6}.
\newblock \DOIprefix\doi{10.1109/IST55454.2022.9827720}.
%Type = Inproceedings
\bibitem[{Arslan and Tsiotras(2015)}]{arslan2015machine}
\bibinfo{author}{Arslan, O.}, \bibinfo{author}{Tsiotras, P.}, \bibinfo{year}{2015}.
\newblock \bibinfo{title}{Machine learning guided exploration for sampling-based motion planning algorithms}, in: \bibinfo{booktitle}{2015 IEEE/RSJ International Conference on Intelligent Robots and Systems (IROS)}, \bibinfo{organization}{IEEE}. pp. \bibinfo{pages}{2646--2652}.
%Type = Inproceedings
\bibitem[{Baldwin and Newman(2010)}]{baldwin2010non}
\bibinfo{author}{Baldwin, I.}, \bibinfo{author}{Newman, P.}, \bibinfo{year}{2010}.
\newblock \bibinfo{title}{Non-parametric learning for natural plan generation}, in: \bibinfo{booktitle}{2010 IEEE/RSJ International Conference on Intelligent Robots and Systems}, \bibinfo{organization}{IEEE}. pp. \bibinfo{pages}{4311--4317}.
%Type = Inproceedings
\bibitem[{Bialkowski et~al.(2013)Bialkowski, Otte and Frazzoli}]{bialkowski2013free}
\bibinfo{author}{Bialkowski, J.}, \bibinfo{author}{Otte, M.}, \bibinfo{author}{Frazzoli, E.}, \bibinfo{year}{2013}.
\newblock \bibinfo{title}{Free-configuration biased sampling for motion planning}, in: \bibinfo{booktitle}{2013 IEEE/RSJ International Conference on Intelligent Robots and Systems}, \bibinfo{organization}{IEEE}. pp. \bibinfo{pages}{1272--1279}.
%Type = Article
\bibitem[{Botev et~al.(2010)Botev, Grotowski and Kroese}]{botev2010kernel}
\bibinfo{author}{Botev, Z.I.}, \bibinfo{author}{Grotowski, J.F.}, \bibinfo{author}{Kroese, D.P.}, \bibinfo{year}{2010}.
\newblock \bibinfo{title}{Kernel density estimation via diffusion}.
\newblock \bibinfo{journal}{The annals of Statistics} \bibinfo{volume}{38}, \bibinfo{pages}{2916--2957}.
%Type = Inproceedings
\bibitem[{Chamzas et~al.(2019)Chamzas, Shrivastava and Kavraki}]{chamzas2019using}
\bibinfo{author}{Chamzas, C.}, \bibinfo{author}{Shrivastava, A.}, \bibinfo{author}{Kavraki, L.E.}, \bibinfo{year}{2019}.
\newblock \bibinfo{title}{Using local experiences for global motion planning}, in: \bibinfo{booktitle}{2019 International Conference on Robotics and Automation (ICRA)}, \bibinfo{organization}{IEEE}. pp. \bibinfo{pages}{8606--8612}.
%Type = Inproceedings
\bibitem[{Choudhury et~al.(2016)Choudhury, Gammell, Barfoot, Srinivasa and Scherer}]{choudhury2016regionally}
\bibinfo{author}{Choudhury, S.}, \bibinfo{author}{Gammell, J.D.}, \bibinfo{author}{Barfoot, T.D.}, \bibinfo{author}{Srinivasa, S.S.}, \bibinfo{author}{Scherer, S.}, \bibinfo{year}{2016}.
\newblock \bibinfo{title}{{Regionally accelerated batch informed trees (RABIT*): A framework to integrate local information into optimal path planning}}, in: \bibinfo{booktitle}{2016 IEEE International Conference on Robotics and Automation (ICRA)}, \bibinfo{organization}{IEEE}. pp. \bibinfo{pages}{4207--4214}.
%Type = Book
\bibitem[{Devroye and Gyorfi(1985)}]{devroye1985nonparametric}
\bibinfo{author}{Devroye, L.}, \bibinfo{author}{Gyorfi, L.}, \bibinfo{year}{1985}.
\newblock \bibinfo{title}{Nonparametric Density Estimation: The $L_1$ View}.
\newblock Wiley Interscience Series in Discrete Mathematics, \bibinfo{publisher}{Wiley}.
%Type = Article
\bibitem[{Dong et~al.(2020)Dong, Zhong and Hong}]{dong2020knowledge}
\bibinfo{author}{Dong, Y.}, \bibinfo{author}{Zhong, Y.}, \bibinfo{author}{Hong, J.}, \bibinfo{year}{2020}.
\newblock \bibinfo{title}{Knowledge-biased sampling-based path planning for automated vehicles parking}.
\newblock \bibinfo{journal}{IEEE Access} \bibinfo{volume}{8}, \bibinfo{pages}{156818--156827}.
%Type = Article
\bibitem[{Enevoldsen et~al.(2022)Enevoldsen, Blanke and Galeazzi}]{enevoldsen2022oe}
\bibinfo{author}{Enevoldsen, T.T.}, \bibinfo{author}{Blanke, M.}, \bibinfo{author}{Galeazzi, R.}, \bibinfo{year}{2022}.
\newblock \bibinfo{title}{Sampling-based collision and grounding avoidance for marine crafts}.
\newblock \bibinfo{journal}{Ocean Engineering} \bibinfo{volume}{261}, \bibinfo{pages}{112078}.
\newblock \DOIprefix\doi{https://doi.org/10.1016/j.oceaneng.2022.112078}.
%Type = Inproceedings
\bibitem[{Enevoldsen and Galeazzi(2021)}]{enevoldsen2021kde}
\bibinfo{author}{Enevoldsen, T.T.}, \bibinfo{author}{Galeazzi, R.}, \bibinfo{year}{2021}.
\newblock \bibinfo{title}{{Grounding-aware RRT* for Path Planning and Safe Navigation of Marine Crafts in Confined Waters}}, in: \bibinfo{booktitle}{13th IFAC Conference on Control Applications in Marine Systems, Robotics, and Vehicles (CAMS) 2021}.
%Type = Inproceedings
\bibitem[{Enevoldsen and Galeazzi(2022)}]{enevoldsen2022iros}
\bibinfo{author}{Enevoldsen, T.T.}, \bibinfo{author}{Galeazzi, R.}, \bibinfo{year}{2022}.
\newblock \bibinfo{title}{Informed sampling-based collision avoidance with least deviation from the nominal path}, in: \bibinfo{booktitle}{2022 IEEE/RSJ International Conference on Intelligent Robots and Systems (IROS)}, pp. \bibinfo{pages}{8094--8100}.
\newblock \DOIprefix\doi{10.1109/IROS47612.2022.9982202}.
%Type = Inproceedings
\bibitem[{Enevoldsen et~al.(2021)Enevoldsen, Reinartz and Galeazzi}]{enevoldsen2021colregs}
\bibinfo{author}{Enevoldsen, T.T.}, \bibinfo{author}{Reinartz, C.}, \bibinfo{author}{Galeazzi, R.}, \bibinfo{year}{2021}.
\newblock \bibinfo{title}{{COLREGs-Informed RRT* for Collision Avoidance of Marine Crafts}}, in: \bibinfo{booktitle}{2021 International Conference on Robotics and Automation (ICRA)}, \bibinfo{organization}{IEEE}.
%Type = Inproceedings
\bibitem[{Gammell et~al.(2014)Gammell, Srinivasa and Barfoot}]{gammell2014informed}
\bibinfo{author}{Gammell, J.D.}, \bibinfo{author}{Srinivasa, S.S.}, \bibinfo{author}{Barfoot, T.D.}, \bibinfo{year}{2014}.
\newblock \bibinfo{title}{{Informed RRT*: Optimal sampling-based path planning focused via direct sampling of an admissible ellipsoidal heuristic}}, in: \bibinfo{booktitle}{2014 IEEE/RSJ International Conference on Intelligent Robots and Systems}, \bibinfo{organization}{IEEE}. pp. \bibinfo{pages}{2997--3004}.
%Type = Inproceedings
\bibitem[{Gammell et~al.(2015)Gammell, Srinivasa and Barfoot}]{gammell2015batch}
\bibinfo{author}{Gammell, J.D.}, \bibinfo{author}{Srinivasa, S.S.}, \bibinfo{author}{Barfoot, T.D.}, \bibinfo{year}{2015}.
\newblock \bibinfo{title}{{Batch informed trees (BIT*): Sampling-based optimal planning via the heuristically guided search of implicit random geometric graphs}}, in: \bibinfo{booktitle}{2015 IEEE international conference on robotics and automation (ICRA)}, \bibinfo{organization}{IEEE}. pp. \bibinfo{pages}{3067--3074}.
%Type = Article
\bibitem[{Gammell and Strub(2021)}]{gammell2021asymptotically}
\bibinfo{author}{Gammell, J.D.}, \bibinfo{author}{Strub, M.P.}, \bibinfo{year}{2021}.
\newblock \bibinfo{title}{Asymptotically optimal sampling-based motion planning methods}.
\newblock \bibinfo{journal}{Annual Review of Control, Robotics, and Autonomous Systems} \bibinfo{volume}{4}, \bibinfo{pages}{295--318}.
%Type = Incollection
\bibitem[{Geraerts and Overmars(2004)}]{geraerts2004comparative}
\bibinfo{author}{Geraerts, R.}, \bibinfo{author}{Overmars, M.H.}, \bibinfo{year}{2004}.
\newblock \bibinfo{title}{A comparative study of probabilistic roadmap planners}, in: \bibinfo{booktitle}{Algorithmic foundations of robotics V}. \bibinfo{publisher}{Springer}, pp. \bibinfo{pages}{43--57}.
%Type = Book
\bibitem[{Gramacki(2018)}]{gramacki2018nonparametric}
\bibinfo{author}{Gramacki, A.}, \bibinfo{year}{2018}.
\newblock \bibinfo{title}{Nonparametric kernel density estimation and its computational aspects}.
\newblock \bibinfo{publisher}{Springer}.
%Type = Article
\bibitem[{Gramacki and Gramacki(2017)}]{gramacki2017fft}
\bibinfo{author}{Gramacki, A.}, \bibinfo{author}{Gramacki, J.}, \bibinfo{year}{2017}.
\newblock \bibinfo{title}{{FFT}-based fast computation of multivariate kernel density estimators with unconstrained bandwidth matrices}.
\newblock \bibinfo{journal}{Journal of Computational and Graphical Statistics} \bibinfo{volume}{26}, \bibinfo{pages}{459--462}.
%Type = Book
\bibitem[{H{\"a}rdle et~al.(2004)H{\"a}rdle, M{\"u}ller, Sperlich, Werwatz et~al.}]{hardle2004nonparametric}
\bibinfo{author}{H{\"a}rdle, W.}, \bibinfo{author}{M{\"u}ller, M.}, \bibinfo{author}{Sperlich, S.}, \bibinfo{author}{Werwatz, A.}, et~al., \bibinfo{year}{2004}.
\newblock \bibinfo{title}{Nonparametric and semiparametric models}. volume~\bibinfo{volume}{1}.
\newblock \bibinfo{publisher}{Springer}.
%Type = Inproceedings
\bibitem[{Hsu et~al.(1997)Hsu, Latombe and Motwani}]{hsu1997path}
\bibinfo{author}{Hsu, D.}, \bibinfo{author}{Latombe, J.C.}, \bibinfo{author}{Motwani, R.}, \bibinfo{year}{1997}.
\newblock \bibinfo{title}{Path planning in expansive configuration spaces}, in: \bibinfo{booktitle}{Proceedings of International Conference on Robotics and Automation}, \bibinfo{organization}{IEEE}. pp. \bibinfo{pages}{2719--2726}.
%Type = Inproceedings
\bibitem[{Ichter et~al.(2018)Ichter, Harrison and Pavone}]{ichter2018learning}
\bibinfo{author}{Ichter, B.}, \bibinfo{author}{Harrison, J.}, \bibinfo{author}{Pavone, M.}, \bibinfo{year}{2018}.
\newblock \bibinfo{title}{Learning sampling distributions for robot motion planning}, in: \bibinfo{booktitle}{2018 IEEE International Conference on Robotics and Automation (ICRA)}, \bibinfo{organization}{IEEE}. pp. \bibinfo{pages}{7087--7094}.
%Type = Article
\bibitem[{Ichter and Pavone(2019)}]{ichter2019robot}
\bibinfo{author}{Ichter, B.}, \bibinfo{author}{Pavone, M.}, \bibinfo{year}{2019}.
\newblock \bibinfo{title}{Robot motion planning in learned latent spaces}.
\newblock \bibinfo{journal}{IEEE Robotics and Automation Letters} \bibinfo{volume}{4}, \bibinfo{pages}{2407--2414}.
%Type = Inproceedings
\bibitem[{Iversen and Ellekilde(2016)}]{iversen2016kernel}
\bibinfo{author}{Iversen, T.F.}, \bibinfo{author}{Ellekilde, L.P.}, \bibinfo{year}{2016}.
\newblock \bibinfo{title}{Kernel density estimation based self-learning sampling strategy for motion planning of repetitive tasks}, in: \bibinfo{booktitle}{2016 IEEE/RSJ International Conference on Intelligent Robots and Systems (IROS)}, \bibinfo{organization}{IEEE}. pp. \bibinfo{pages}{1380--1387}.
%Type = Article
\bibitem[{Janson et~al.(2015)Janson, Schmerling, Clark and Pavone}]{janson2015fast}
\bibinfo{author}{Janson, L.}, \bibinfo{author}{Schmerling, E.}, \bibinfo{author}{Clark, A.}, \bibinfo{author}{Pavone, M.}, \bibinfo{year}{2015}.
\newblock \bibinfo{title}{Fast marching tree: A fast marching sampling-based method for optimal motion planning in many dimensions}.
\newblock \bibinfo{journal}{The International journal of robotics research} \bibinfo{volume}{34}, \bibinfo{pages}{883--921}.
%Type = Incollection
\bibitem[{Janson et~al.(2018)Janson, Schmerling and Pavone}]{janson2018monte}
\bibinfo{author}{Janson, L.}, \bibinfo{author}{Schmerling, E.}, \bibinfo{author}{Pavone, M.}, \bibinfo{year}{2018}.
\newblock \bibinfo{title}{Monte carlo motion planning for robot trajectory optimization under uncertainty}, in: \bibinfo{booktitle}{Robotics Research}. \bibinfo{publisher}{Springer}, pp. \bibinfo{pages}{343--361}.
%Type = Inproceedings
\bibitem[{Joshi and Panagiotis(2019)}]{joshi2019non}
\bibinfo{author}{Joshi, S.S.}, \bibinfo{author}{Panagiotis, T.}, \bibinfo{year}{2019}.
\newblock \bibinfo{title}{Non-parametric informed exploration for sampling-based motion planning}, in: \bibinfo{booktitle}{2019 International Conference on Robotics and Automation (ICRA)}, \bibinfo{organization}{IEEE}. pp. \bibinfo{pages}{5915--5921}.
%Type = Article
\bibitem[{Karaman and Frazzoli(2011)}]{karaman2011sampling}
\bibinfo{author}{Karaman, S.}, \bibinfo{author}{Frazzoli, E.}, \bibinfo{year}{2011}.
\newblock \bibinfo{title}{Sampling-based algorithms for optimal motion planning}.
\newblock \bibinfo{journal}{The international journal of robotics research} \bibinfo{volume}{30}, \bibinfo{pages}{846--894}.
%Type = Article
\bibitem[{Kavraki et~al.(1998)Kavraki, Kolountzakis and Latombe}]{kavraki1998analysis}
\bibinfo{author}{Kavraki, L.E.}, \bibinfo{author}{Kolountzakis, M.N.}, \bibinfo{author}{Latombe, J.C.}, \bibinfo{year}{1998}.
\newblock \bibinfo{title}{Analysis of probabilistic roadmaps for path planning}.
\newblock \bibinfo{journal}{IEEE Transactions on Robotics and automation} \bibinfo{volume}{14}, \bibinfo{pages}{166--171}.
%Type = Article
\bibitem[{Kavraki et~al.(1996)Kavraki, Svestka, Latombe and Overmars}]{kavraki1996probabilistic}
\bibinfo{author}{Kavraki, L.E.}, \bibinfo{author}{Svestka, P.}, \bibinfo{author}{Latombe, J.C.}, \bibinfo{author}{Overmars, M.H.}, \bibinfo{year}{1996}.
\newblock \bibinfo{title}{Probabilistic roadmaps for path planning in high-dimensional configuration spaces}.
\newblock \bibinfo{journal}{IEEE transactions on Robotics and Automation} \bibinfo{volume}{12}, \bibinfo{pages}{566--580}.
%Type = Inproceedings
\bibitem[{Kim et~al.(2018)Kim, Kwon and Yoon}]{kim2018dancing}
\bibinfo{author}{Kim, D.}, \bibinfo{author}{Kwon, Y.}, \bibinfo{author}{Yoon, S.E.}, \bibinfo{year}{2018}.
\newblock \bibinfo{title}{{Dancing PRM*: Simultaneous planning of sampling and optimization with configuration free space approximation}}, in: \bibinfo{booktitle}{2018 IEEE International Conference on Robotics and Automation (ICRA)}, \bibinfo{organization}{IEEE}. pp. \bibinfo{pages}{7071--7078}.
%Type = Inproceedings
\bibitem[{Klemm et~al.(2015)Klemm, Oberl{\"a}nder, Hermann, Roennau, Schamm, Zollner and Dillmann}]{klemm2015rrt}
\bibinfo{author}{Klemm, S.}, \bibinfo{author}{Oberl{\"a}nder, J.}, \bibinfo{author}{Hermann, A.}, \bibinfo{author}{Roennau, A.}, \bibinfo{author}{Schamm, T.}, \bibinfo{author}{Zollner, J.M.}, \bibinfo{author}{Dillmann, R.}, \bibinfo{year}{2015}.
\newblock \bibinfo{title}{{RRT*-Connect: Faster, asymptotically optimal motion planning}}, in: \bibinfo{booktitle}{2015 IEEE international conference on robotics and biomimetics (ROBIO)}, \bibinfo{organization}{IEEE}. pp. \bibinfo{pages}{1670--1677}.
%Type = Inproceedings
\bibitem[{Kuffner and LaValle(2000)}]{kuffner2000rrt}
\bibinfo{author}{Kuffner, J.J.}, \bibinfo{author}{LaValle, S.M.}, \bibinfo{year}{2000}.
\newblock \bibinfo{title}{{RRT-connect: An efficient approach to single-query path planning}}, in: \bibinfo{booktitle}{Proceedings 2000 ICRA. Millennium Conference. IEEE International Conference on Robotics and Automation. Symposia Proceedings (Cat. No. 00CH37065)}, \bibinfo{organization}{IEEE}. pp. \bibinfo{pages}{995--1001}.
%Type = Inproceedings
\bibitem[{Kunz et~al.(2016)Kunz, Thomaz and Christensen}]{kunz2016hierarchical}
\bibinfo{author}{Kunz, T.}, \bibinfo{author}{Thomaz, A.}, \bibinfo{author}{Christensen, H.}, \bibinfo{year}{2016}.
\newblock \bibinfo{title}{Hierarchical rejection sampling for informed kinodynamic planning in high-dimensional spaces}, in: \bibinfo{booktitle}{2016 IEEE International Conference on Robotics and Automation (ICRA)}, \bibinfo{organization}{IEEE}. pp. \bibinfo{pages}{89--96}.
%Type = Article
\bibitem[{LaValle(1998)}]{lavalle1998rapidly}
\bibinfo{author}{LaValle, S.M.}, \bibinfo{year}{1998}.
\newblock \bibinfo{title}{Rapidly-exploring random trees : a new tool for path planning}.
\newblock \bibinfo{journal}{The annual research report} .
%Type = Article
\bibitem[{LaValle and Kuffner(2001)}]{LaValle2001}
\bibinfo{author}{LaValle, S.M.}, \bibinfo{author}{Kuffner, J.J.}, \bibinfo{year}{2001}.
\newblock \bibinfo{title}{{Randomized Kinodynamic Planning}}.
\newblock \bibinfo{journal}{The International Journal of Robotics Research} \bibinfo{volume}{20}, \bibinfo{pages}{378--400}.
%Type = Inproceedings
\bibitem[{Lehner and Albu-Sch{\"a}ffer(2017)}]{lehner2017repetition}
\bibinfo{author}{Lehner, P.}, \bibinfo{author}{Albu-Sch{\"a}ffer, A.}, \bibinfo{year}{2017}.
\newblock \bibinfo{title}{Repetition sampling for efficiently planning similar constrained manipulation tasks}, in: \bibinfo{booktitle}{2017 IEEE/RSJ International Conference on Intelligent Robots and Systems (IROS)}, \bibinfo{organization}{IEEE}. pp. \bibinfo{pages}{2851--2856}.
%Type = Article
\bibitem[{Lehner and Albu-Sch{\"a}ffer(2018)}]{lehner2018repetition}
\bibinfo{author}{Lehner, P.}, \bibinfo{author}{Albu-Sch{\"a}ffer, A.}, \bibinfo{year}{2018}.
\newblock \bibinfo{title}{The repetition roadmap for repetitive constrained motion planning}.
\newblock \bibinfo{journal}{IEEE Robotics and Automation Letters} \bibinfo{volume}{3}, \bibinfo{pages}{3884--3891}.
%Type = Article
\bibitem[{Li et~al.(2021)Li, Miao, Qureshi and Yip}]{li2021mpc}
\bibinfo{author}{Li, L.}, \bibinfo{author}{Miao, Y.}, \bibinfo{author}{Qureshi, A.H.}, \bibinfo{author}{Yip, M.C.}, \bibinfo{year}{2021}.
\newblock \bibinfo{title}{{MPC-MPNet: Model-Predictive Motion Planning Networks for Fast, Near-Optimal Planning under Kinodynamic Constraints}}.
\newblock \bibinfo{journal}{IEEE Robotics and Automation Letters} \bibinfo{volume}{6}, \bibinfo{pages}{4496--4503}.
%Type = Inproceedings
\bibitem[{Lindemann and LaValle(2005)}]{lindemann2005current}
\bibinfo{author}{Lindemann, S.R.}, \bibinfo{author}{LaValle, S.M.}, \bibinfo{year}{2005}.
\newblock \bibinfo{title}{Current issues in sampling-based motion planning}, in: \bibinfo{booktitle}{Robotics research. The eleventh international symposium}, \bibinfo{organization}{Springer}. pp. \bibinfo{pages}{36--54}.
%Type = Article
\bibitem[{Mandalika et~al.(2021)Mandalika, Scalise, Hou, Choudhury and Srinivasa}]{mandalika2021guided}
\bibinfo{author}{Mandalika, A.}, \bibinfo{author}{Scalise, R.}, \bibinfo{author}{Hou, B.}, \bibinfo{author}{Choudhury, S.}, \bibinfo{author}{Srinivasa, S.S.}, \bibinfo{year}{2021}.
\newblock \bibinfo{title}{Guided incremental local densification for accelerated sampling-based motion planning}.
\newblock \bibinfo{journal}{arXiv preprint arXiv:2104.05037} .
%Type = Misc
\bibitem[{Odland(2018)}]{odland2018kdepy}
\bibinfo{author}{Odland, T.}, \bibinfo{year}{2018}.
\newblock \bibinfo{title}{{tommyod/KDEpy: Kernel Density Estimation in Python}}.
\newblock \DOIprefix\doi{10.5281/zenodo.2392268}.
%Type = Article
\bibitem[{Osa(2020)}]{osa2020multimodal}
\bibinfo{author}{Osa, T.}, \bibinfo{year}{2020}.
\newblock \bibinfo{title}{Multimodal trajectory optimization for motion planning}.
\newblock \bibinfo{journal}{The International Journal of Robotics Research} \bibinfo{volume}{39}, \bibinfo{pages}{983--1001}.
%Type = Article
\bibitem[{Otte and Frazzoli(2016)}]{otte2016rrtx}
\bibinfo{author}{Otte, M.}, \bibinfo{author}{Frazzoli, E.}, \bibinfo{year}{2016}.
\newblock \bibinfo{title}{{RRT$^X$: Asymptotically optimal single-query sampling-based motion planning with quick replanning}}.
\newblock \bibinfo{journal}{The International Journal of Robotics Research} \bibinfo{volume}{35}, \bibinfo{pages}{797--822}.
%Type = Article
\bibitem[{O’Brien et~al.(2016)O’Brien, Kashinath, Cavanaugh, Collins and O’Brien}]{o2016fast}
\bibinfo{author}{O’Brien, T.A.}, \bibinfo{author}{Kashinath, K.}, \bibinfo{author}{Cavanaugh, N.R.}, \bibinfo{author}{Collins, W.D.}, \bibinfo{author}{O’Brien, J.P.}, \bibinfo{year}{2016}.
\newblock \bibinfo{title}{{A fast and objective multidimensional kernel density estimation method: fastKDE}}.
\newblock \bibinfo{journal}{Computational Statistics \& Data Analysis} \bibinfo{volume}{101}, \bibinfo{pages}{148--160}.
%Type = Article
\bibitem[{Schmerling and Pavone(2016)}]{schmerling2016evaluating}
\bibinfo{author}{Schmerling, E.}, \bibinfo{author}{Pavone, M.}, \bibinfo{year}{2016}.
\newblock \bibinfo{title}{Evaluating trajectory collision probability through adaptive importance sampling for safe motion planning}.
\newblock \bibinfo{journal}{arXiv preprint arXiv:1609.05399} .
%Type = Book
\bibitem[{Scott(2015)}]{scott2015multivariate}
\bibinfo{author}{Scott, D.W.}, \bibinfo{year}{2015}.
\newblock \bibinfo{title}{Multivariate density estimation: theory, practice, and visualization}.
\newblock \bibinfo{publisher}{John Wiley \& Sons}.
%Type = Article
\bibitem[{Silverman(1982)}]{silverman1982algorithm}
\bibinfo{author}{Silverman, B.W.}, \bibinfo{year}{1982}.
\newblock \bibinfo{title}{Algorithm as 176: Kernel density estimation using the fast fourier transform}.
\newblock \bibinfo{journal}{Journal of the Royal Statistical Society. Series C (Applied Statistics)} \bibinfo{volume}{31}, \bibinfo{pages}{93--99}.
%Type = Inproceedings
\bibitem[{Strub and Gammell(2020a)}]{strub2020adaptively}
\bibinfo{author}{Strub, M.P.}, \bibinfo{author}{Gammell, J.D.}, \bibinfo{year}{2020}a.
\newblock \bibinfo{title}{{Adaptively Informed Trees (AIT*): Fast asymptotically optimal path planning through adaptive heuristics}}, in: \bibinfo{booktitle}{2020 IEEE International Conference on Robotics and Automation (ICRA)}, \bibinfo{organization}{IEEE}. pp. \bibinfo{pages}{3191--3198}.
%Type = Inproceedings
\bibitem[{Strub and Gammell(2020b)}]{strub2020advanced}
\bibinfo{author}{Strub, M.P.}, \bibinfo{author}{Gammell, J.D.}, \bibinfo{year}{2020}b.
\newblock \bibinfo{title}{{Advanced BIT* (ABIT*): Sampling-based planning with advanced graph-search techniques}}, in: \bibinfo{booktitle}{2020 IEEE International Conference on Robotics and Automation (ICRA)}, \bibinfo{organization}{IEEE}. pp. \bibinfo{pages}{130--136}.
%Type = Article
\bibitem[{V{\'e}ras et~al.(2019)V{\'e}ras, Medeiros and Guimar{\'a}es}]{veras2019systematic}
\bibinfo{author}{V{\'e}ras, L.G.D.}, \bibinfo{author}{Medeiros, F.L.}, \bibinfo{author}{Guimar{\'a}es, L.N.}, \bibinfo{year}{2019}.
\newblock \bibinfo{title}{Systematic literature review of sampling process in rapidly-exploring random trees}.
\newblock \bibinfo{journal}{IEEE Access} \bibinfo{volume}{7}, \bibinfo{pages}{50933--50953}.
%Type = Article
\bibitem[{Wang et~al.(2020)Wang, Chi, Li, Wang and Meng}]{wang2020neural}
\bibinfo{author}{Wang, J.}, \bibinfo{author}{Chi, W.}, \bibinfo{author}{Li, C.}, \bibinfo{author}{Wang, C.}, \bibinfo{author}{Meng, M.Q.H.}, \bibinfo{year}{2020}.
\newblock \bibinfo{title}{{Neural RRT*: Learning-based optimal path planning}}.
\newblock \bibinfo{journal}{IEEE Transactions on Automation Science and Engineering} \bibinfo{volume}{17}, \bibinfo{pages}{1748--1758}.
%Type = Inproceedings
\bibitem[{Yi et~al.(2015)Yi, Goodrich and Seppi}]{yi2015morrf}
\bibinfo{author}{Yi, D.}, \bibinfo{author}{Goodrich, M.A.}, \bibinfo{author}{Seppi, K.D.}, \bibinfo{year}{2015}.
\newblock \bibinfo{title}{{MORRF*: Sampling-based multi-objective motion planning}}, in: \bibinfo{booktitle}{Twenty-Fourth International Joint Conference on Artificial Intelligence}.
%Type = Inproceedings
\bibitem[{Yi et~al.(2018)Yi, Thakker, Gulino, Salzman and Srinivasa}]{yi2018generalizing}
\bibinfo{author}{Yi, D.}, \bibinfo{author}{Thakker, R.}, \bibinfo{author}{Gulino, C.}, \bibinfo{author}{Salzman, O.}, \bibinfo{author}{Srinivasa, S.}, \bibinfo{year}{2018}.
\newblock \bibinfo{title}{Generalizing informed sampling for asymptotically-optimal sampling-based kinodynamic planning via markov chain monte carlo}, in: \bibinfo{booktitle}{2018 IEEE International Conference on Robotics and Automation (ICRA)}, \bibinfo{organization}{IEEE}. pp. \bibinfo{pages}{7063--7070}.
%Type = Inproceedings
\bibitem[{Zhang et~al.(2018)Zhang, Huh and Lee}]{zhang2018learning}
\bibinfo{author}{Zhang, C.}, \bibinfo{author}{Huh, J.}, \bibinfo{author}{Lee, D.D.}, \bibinfo{year}{2018}.
\newblock \bibinfo{title}{Learning implicit sampling distributions for motion planning}, in: \bibinfo{booktitle}{2018 IEEE/RSJ International Conference on Intelligent Robots and Systems (IROS)}, \bibinfo{organization}{IEEE}. pp. \bibinfo{pages}{3654--3661}.

\end{thebibliography}

\end{document}